\pdfoutput=1
\documentclass[10pt, logo, onecolumn, copyright]{nv}

\usepackage[utf8]{inputenc}      % UTF-8 input
\usepackage[T1]{fontenc}         % 8-bit T1 fonts
\usepackage{times}               % Times Roman body font (SANA style)
\usepackage{inconsolata}         % monospaced font
\usepackage{microtype}           % microtypography

\usepackage{amsmath}
\usepackage{amssymb}
\usepackage{amsfonts}            % blackboard math symbols
\usepackage{amsthm}
\usepackage{mathtools}
\usepackage{bm}
\usepackage{bbm}
\usepackage{nicefrac}            % compact symbols for 1/2, etc.

\usepackage[table]{xcolor}       % colours (table option pulls in colortbl)
\usepackage{booktabs}            % professional-quality tables
\usepackage{multirow}
\usepackage{multicol}

\usepackage{graphicx}
\usepackage{wrapfig}
\usepackage{float}
\usepackage{tikz}
\usepackage{pgfplots}
\pgfplotsset{compat=1.18}
\usepgfplotslibrary{fillbetween}
\usepgflibrary{patterns}
\usetikzlibrary{arrows.meta,positioning,calc,shapes.geometric,decorations.text,
  patterns,patterns.meta,fit,backgrounds,chains,shadows,math,circuits.ee.IEC,
  decorations.pathmorphing,decorations.pathreplacing,decorations.shapes}

\usepackage[labelfont=bf,font=small]{caption}
\usepackage{subcaption}          % replaces the deprecated 'subfigure'
\usepackage{listings}
\usepackage{minted}              % NOTE: requires pdflatex -shell-escape + Pygments

\usepackage{url}
\usepackage{textcomp}
\usepackage{pifont}
\usepackage{scalefnt}
\usepackage{tcolorbox}
\usepackage{fontawesome}
\usepackage{lipsum}              % placeholder text only -- remove for final

\usepackage[square,sort,comma,numbers]{natbib}
\let\cite\citep                  % route every \cite through \citep

\usepackage{hyperref}
\definecolor{nvidiagreen}{HTML}{76B900}
\hypersetup{
  colorlinks=true,
  breaklinks=true,
  pdfusetitle=true,
  urlcolor=nvidiagreen,
  linkcolor=nvidiagreen,
  citecolor=nvidiagreen,
}
\definecolor{codebg}{RGB}{245,245,245}
\definecolor{keywordcolor}{RGB}{0,0,153}
\definecolor{commentcolor}{RGB}{34,139,34}
\definecolor{stringcolor}{RGB}{163,21,21}
\definecolor{numbercolor}{RGB}{128,128,128}
\definecolor{brickred}{HTML}{b92622}
\definecolor{midnightblue}{HTML}{005c7f}
\definecolor{salmon}{HTML}{f1958d}
\definecolor{burntorange}{HTML}{f19249}
\definecolor{junglegreen}{HTML}{4dae9d}
\definecolor{forestgreen}{HTML}{499c5e}
\definecolor{pinegreen}{HTML}{3d8a75}
\definecolor{seagreen}{HTML}{6bc1a2}
\definecolor{limegreen}{HTML}{97c65a}
\definecolor{violet}{HTML}{8f00ff}
\definecolor{pastelviolet}{HTML}{cb99c9}
\definecolor{darkcyan}{HTML}{008B8B}

\usepackage{amsmath,amsfonts,bm}

\def\eqref#1{equation~\ref{#1}}
\def\1{\bm{1}}

\DeclareMathAlphabet{\mathsfit}{\encodingdefault}{\sfdefault}{m}{sl}
\SetMathAlphabet{\mathsfit}{bold}{\encodingdefault}{\sfdefault}{bx}{n}

\newcommand{\softmax}{\mathrm{softmax}}

\newcommand{\KL}{D_{\mathrm{KL}}}

\definecolor{sagelight}{HTML}{E3E9DB}
\definecolor{greensoft}{HTML}{B2CF81}
\definecolor{greenfresh}{HTML}{9BC653}
\definecolor{nvgreen}{HTML}{87B737}
\definecolor{greenmoss}{HTML}{71913B}
\definecolor{greenolive}{HTML}{658929}
\definecolor{greendeep}{HTML}{516E21}
\definecolor{greendark}{HTML}{3F541C}
\definecolor{lilacpale}{HTML}{CDD0EB}
\definecolor{lilacsoft}{HTML}{9FA7D0}
\definecolor{lilacblue}{HTML}{737EAE}
\definecolor{slateblue}{HTML}{4C567D}
\definecolor{steelblue}{HTML}{536A82}
\definecolor{navydark}{HTML}{2C455B}
\colorlet{cellbest}{nvgreen!65}      % table cell background, best result (salient, matches figure green)
\colorlet{cellsecond}{lilacpale!80}  % table cell background, second-best result (subtle)
\colorlet{linkaccent}{slateblue}     % hyperref link and URL color
\colorlet{citeaccent}{nvgreen!85!greendark} % citation-number color, vivid but not light
\colorlet{promptframe}{slateblue}    % prompt-box frame
\colorlet{prompttitle}{lilacpale!40} % prompt-box title band
\newtcolorbox{promptbox}[1]{
  enhanced, breakable,
  colback=white,
  colframe=promptframe,
  colbacktitle=prompttitle,
  coltitle=black,
  fonttitle=\small\bfseries,
  fontupper=\small\ttfamily,
  title={#1},
  boxrule=0.6pt,
  arc=2pt,
  left=6pt, right=6pt, top=5pt, bottom=5pt,
}
\colorlet{promptph}{lilacpale!75}    % placeholder chip, e.g. {task description}
\colorlet{prompttag}{greendeep}      % tags the model must emit, e.g. <think>
\colorlet{promptcardbg}{lilacpale!18} % embedded JSON schema, output block, or calibration example
\newcommand{\ph}[1]{{\setlength{\fboxsep}{1pt}\colorbox{promptph}{\{#1\}}}}
\newcommand{\ptag}[1]{\textcolor{prompttag}{\textbf{#1}}}
\newcommand{\promptsep}{\par\vspace{3pt}\noindent\tikz\draw[promptframe!45, dash pattern=on 2pt off 2pt, line width=0.4pt] (0,0) -- (\linewidth,0);\par\vspace{3pt}}

\newtcolorbox{promptcard}{enhanced, colback=promptcardbg, frame hidden, arc=2pt, left=4pt, right=4pt, top=2pt, bottom=2pt, before skip=3pt, after skip=3pt, fontupper=\small\ttfamily}
\colorlet{excard}{lilacpale!22}      % background of one turn card
\colorlet{exbadge}{lilacblue}        % turn-number badge
\colorlet{exaction}{greendeep}       % [Action] label
\colorlet{exobs}{slateblue}          % [Observation] label
\colorlet{exhighlight}{greensoft!55} % evidence that decides the next action
\colorlet{exsuccess}{nvgreen}        % success banner
\newtcolorbox{examplebox}[1]{
  enhanced,
  colback=white, colframe=promptframe, colbacktitle=prompttitle, coltitle=black,
  fonttitle=\small\bfseries, fontupper=\small, title={#1},
  boxrule=0.6pt, arc=2pt, left=6pt, right=6pt, top=5pt, bottom=5pt,
}
\newcommand{\extask}[2]{\noindent\textbf{#1:} #2\par\vspace{3pt}}
\newcommand{\exinit}[1]{\noindent{\footnotesize\textcolor{exobs}{\textbf{[Initial observation]}}~#1}\par\vspace{2pt}}
\newcommand{\exbadge}[1]{\tikz[baseline=(b.base)]\node[circle, fill=exbadge, text=white, font=\scriptsize\bfseries, inner sep=0pt, minimum size=1.45em](b){#1};}
\newcommand{\exturn}[4][Observation]{%
  \begin{tcolorbox}[enhanced, colback=excard, frame hidden, arc=2pt, left=3pt, right=4pt, top=1.5pt, bottom=1.5pt, before skip=2pt, after skip=2pt]%
  \exbadge{#2}\hspace{5pt}%
  \parbox[t]{\dimexpr\linewidth-1.45em-5pt\relax}{\raggedright%
    \textcolor{exaction}{\textbf{[Action]}}~\texttt{#3}%
    \if\relax\detokenize{#4}\relax\else\\[1pt]{\footnotesize\textcolor{exobs}{\textbf{[#1]}}~#4}\fi}%
  \end{tcolorbox}}
\newcommand{\exhl}[1]{{\setlength{\fboxsep}{1pt}\colorbox{exhighlight}{#1}}}
\newcommand{\exoutcome}[1]{%
  \begin{tcolorbox}[enhanced, colback=exsuccess, frame hidden, arc=2pt, halign=center, top=2pt, bottom=2pt, before skip=5pt, after skip=0pt, fontupper=\small\bfseries\color{white}]%
  \checkmark~Success: #1%
  \end{tcolorbox}}

\usepackage{algorithm}       % floating algorithm environment
\usepackage{algpseudocode}   % \Procedure, \State, \For, \If, \Call, \Comment
\tcbuselibrary{most}         % enhanced, breakable prompt boxes in the appendix

\DeclareCaptionFont{eightpt}{\fontsize{8}{9.5}\selectfont}
\AddToHook{env/wrapfigure/begin}{\captionsetup{font=eightpt}}
\AddToHook{env/wraptable/begin}{\captionsetup{font=eightpt}}

\setlist{topsep=2pt, partopsep=0pt, itemsep=3pt, parsep=0pt}

\newcommand{\na}{--}
\newcommand{\method}{\texorpdfstring{\mbox{\textsc{\textls[0]{PivotOPD}}}}{PivotOPD}}

\newtheorem{lemma}{Lemma}
\newtheorem{proposition}{Proposition}

\newcommand{\pib}{\pi_{\bar\theta}}
\newcommand{\act}{\operatorname{act}}

\usepackage{cleveref}
\crefname{equation}{Eq.}{Eqs.}
\Crefname{equation}{Eq.}{Eqs.}
\creflabelformat{equation}{#2#1#3}
\crefname{algorithm}{algorithm}{algorithms}
\Crefname{algorithm}{Algorithm}{Algorithms}
\crefname{proposition}{proposition}{propositions}
\Crefname{proposition}{Proposition}{Propositions}
\crefname{lemma}{lemma}{lemmas}
\Crefname{lemma}{Lemma}{Lemmas}
\crefname{appendix}{appendix}{appendices}
\Crefname{appendix}{Appendix}{Appendices}

\title{%
\centering
\method{}: Learning to Recover from\\
Pivotal Mistakes in Multi-Turn Agents%
}


\hypersetup{
  pdftitle={PivotOPD: Learning to Recover from Pivotal Mistakes in Multi-Turn Agents},
  pdfauthor={Yinghui He, Yapei Chang, Khushi Bhardwaj, Daniele Molinari, Tugrul Konuk, Jan Kautz, Ali Hatamizadeh},
}

\author{%
\vspace{-1.5em}
\centering
\fontsize{10.5pt}{18pt}\selectfont
Yinghui~He$^{1,2*\dagger}$ ~~ Yapei~Chang$^{2,3}$ ~~ Khushi~Bhardwaj$^{2}$ ~~ Daniele~Molinari$^{2}$
\\
\bfseries\fontsize{10.5pt}{18pt}\selectfont
Tugrul~Konuk$^{2}$ ~~ Jan~Kautz$^{2}$ ~~ Ali~Hatamizadeh$^{2\dagger}$
\\
\vspace{1.3mm}
{\small $^{1}$Princeton University ~~~ $^{2}$NVIDIA
~~~ $^{3}$University of Maryland
}
\\
\vspace{1.3mm}
{\tt\small yh0068@princeton.edu, ~ahatamizadeh@nvidia.com}
\\
\vspace{1.3mm}
{\small \textcolor{nvidiagreen}{\faGlobe}~~Project page: \href{https://research.nvidia.com/labs/lpr/pivotopd/}{\nolinkurl{https://research.nvidia.com/labs/lpr/pivotopd/}}}
\vspace{0.3em}
}

\correspondingauthor={yh0068@princeton.edu, ahatamizadeh@nvidia.com}

\begin{abstract}
\noindent \textbf{Abstract:} On-policy distillation (OPD) is a promising approach for training language agents, providing dense teacher supervision on student-generated trajectories. However, in multi-turn interaction, an incorrect action changes the states the student encounters later, so errors compound across turns. In our preliminary experiments across three Qwen3 models (8B--235B), we find that more than half of the failed rollouts contain a \textit{pivotal mistake}, an action that moves the agent farther from completing the task, and this mistake typically occurs early. These pivotal mistakes often remain recoverable: guiding the model for only a few turns after the pivotal turn can restore task success.
We therefore propose \textbf{\method{}}, an on-policy distillation framework that jointly trains the student to prevent pivotal mistakes and to recover from the states they create. At each pivotal mistake, a teacher model provides a gold action and then names a recovery action at each of the next few turns. \textbf{\textit{Preventive distillation}} uses the gold action with reverse KL to steer the student away from the pivotal mistake, while \textbf{\textit{recovery distillation}} uses the recovery actions with forward KL to transfer recovery behaviors that the student rarely samples.
Against $13$ baselines on ALFWorld, WebShop, and Search-based QA, \method{} achieves the strongest average performance for both Qwen3-1.7B and Qwen3-8B students, improving over the strongest baseline on ALFWorld by $+5.5\%$ with the 1.7B student. The gains also transfer to another model family on the software engineering domain, where \method{} raises the resolve rate of a Nemotron-3.5 student on SWE-Bench Verified by $+3.2\%$.
\end{abstract}

\begin{document}
\maketitle

\begin{figure}[H]
\centering
\includegraphics[width=0.92\linewidth]{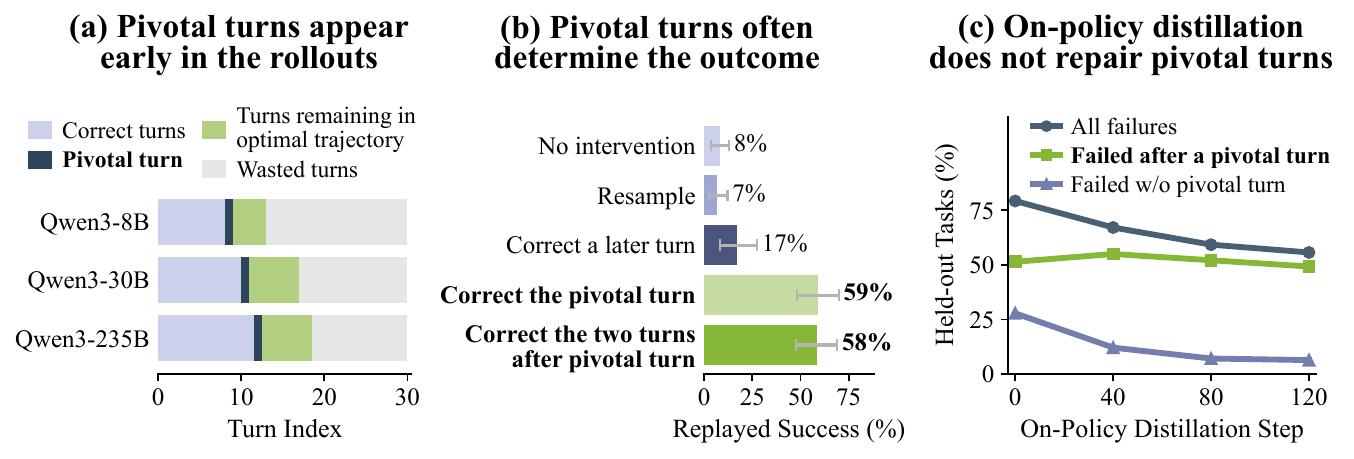}
\vspace{3pt}
\caption{\small \textbf{Failed rollouts trace back to an early pivotal mistake, and standard on-policy distillation does not repair it.} \textbf{(a)} Pivotal turns arrive early; the remaining optimal trajectory is short, yet students waste the turns until the end. \textbf{(b)} Correcting the pivotal turn or guiding recovery turns both restore success effectively. \textbf{(c)} OPD eliminates most incomplete-trajectory failures, but pivotal-turn failures persist across training.}
\label{fig:motivation}
% \vspace{-15pt}
\end{figure}

\section{Introduction}
Language agents are increasingly capable of solving complex tasks through multi-turn interactions \citep{yao2022react,liu2023agentbench,zhou2024webarena,yang2024sweagent,yao2024taubench,shridhar2020alfworld,yao2022webshop,jin2025searchr1}. To improve their performance, recent work has adopted on-policy distillation (OPD) as an effective training paradigm \citep{zhao2026opsd,zhou2026turnopd}, since it provides dense, token-level teacher supervision on trajectories that the student generates itself \citep{agarwal2024gkd,gu2024minillm,lu2025onpolicydistillation}. However, OPD is more challenging in multi-turn settings because each action changes the external environment, which then returns new observations and determines which actions are available in later turns \citep{ross2011reduction,wang2026tcod}. A single incorrect action can lead to \emph{error accumulation}, in which the student ends up in states created by its own mistake and diverges further from the teacher over the turns that follow \citep{ross2010efficient,ross2011reduction,wang2026tcod,zhong2026sod}. Existing multi-turn OPD methods address error accumulation by reweighting turns or restricting which turns receive the distillation signal \citep{zhou2026turnopd,zhong2026sod,wang2026tcod,zhang2026stepopsd}, but leave open whether such failures hinge on a single decisive mistake and whether the student can still recover from it.

\begin{figure}[t]
\centering
\includegraphics[width=0.92\linewidth]{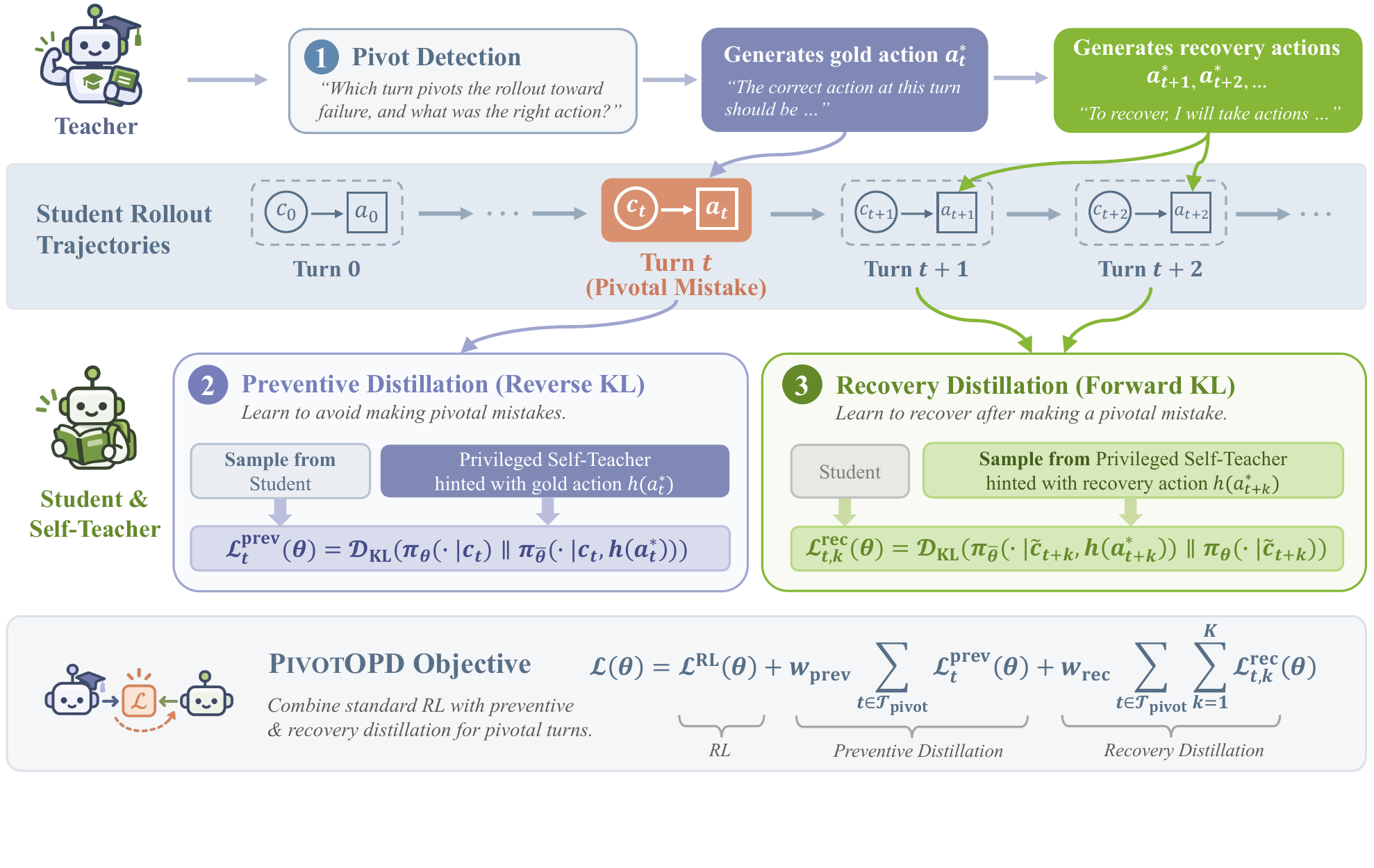}
\vspace{-20pt}
\caption{\small \textbf{Overview of \method{}.}
\textbf{(1) Pivot detection.} A teacher model reads each rollout with its outcome, selects candidate turns, and names a gold action at each. A candidate turn is pivotal when the student's committed action differs from the gold action. After each pivotal turn, the teacher names a recovery action at each of the next few turns.
\textbf{(2) Preventive distillation} trains the student to avoid the pivotal mistake. The frozen student hinted with the gold action serves as a privileged self-teacher that re-scores the student's recorded response, and a reverse KL loss moves the student toward the gold action.
\textbf{(3) Recovery distillation} trains the student to recover after the mistake. The self-teacher is instead hinted with the recovery action and writes a recovery response, on which the student is trained without the hint through a forward KL loss. Later recovery turns start from the state reached by executing its action in a copy of the environment that replays all preceding actions.
Both terms are combined with group-based RL in a single PPO update.}
\label{fig:method}
\end{figure}

To answer these questions, we analyze failed rollouts in ALFWorld \citep{shridhar2020alfworld}, where the agent completes household goals, such as heating an egg, through textual actions (\Cref{app:benchmark_examples}), and summarize the results in \Cref{fig:motivation}. Its symbolic oracle reads the full environment state and computes the remaining optimal trajectory, i.e., the shortest action sequence that completes the task, at every turn. Across three Qwen3 models (8B--235B) \citep{yang2025qwen3}, more than half of the failed rollouts contain a \emph{pivotal mistake}: an action that lengthens the remaining optimal trajectory or makes the task unsolvable. We call the turn where it is committed a \emph{pivotal turn}. The first one typically arrives early, and all three models waste most of the remaining turns without recovering (\Cref{fig:motivation}a). In counterfactual replays of Qwen3-8B's failures, correcting the pivotal turn with the oracle action raises the replayed success from $8\%$ to $59\%$ (\Cref{fig:motivation}b). To test whether a failure can be repaired after the mistake, we keep the mistake and apply the oracle action at the next two turns, which still reaches $58\%$. Therefore, pivotal mistakes are largely recoverable.
% , and these failures mainly reflect an inability to recover rather than to complete the task.

Existing training methods rarely teach this recovery, because they learn mostly from the student's own trajectories, and these rarely contain a recovery. Outcome-based RL rewards a recovery only when a rollout happens to find one, and its group-relative advantage is zero whenever every rollout of a task fails \citep{shao2024deepseekmath}. Methods that assign credit to individual turns or focus on the turns that decide the outcome share this limitation \citep{feng2025gigpo,yi2026pivotrl,yang2026opid}, and so does OPD. In our analysis, standard OPD lowers the failure rate from $79\%$ to $56\%$ of held-out tasks, while failures after a pivotal turn only fall from $51\%$ to $49\%$ (\Cref{fig:motivation}c). Although OPD reduces the probability of the committed mistake, the oracle action still has less than $1\%$ probability at every evaluated pivotal turn (\Cref{app:analysis}). The student therefore lacks a direct learning signal at the states its own mistakes create.

To provide this signal, we introduce \textbf{\method{}}, which augments group-based RL with dense token-level supervision concentrated at pivotal turns and the turns that follow them (\Cref{fig:method}). Since most environments offer no oracle, \method{} estimates pivotal turns through \emph{pivot detection}: a teacher model, a larger LLM in our main setting, reads each rollout in hindsight, selects a few candidate turns where the student may have gone wrong, and names a gold action at each. We treat a candidate turn as pivotal when the student's committed action differs from the gold action. On ALFWorld, at least one detected pivotal turn falls within one turn of the oracle-labeled one in $77.8\%$ of failed rollouts on average, at least twice the rate of randomly chosen turns (\Cref{tab:teacher_oracle}).

\method{} obtains its token-level targets from a \emph{privileged self-teacher}, the frozen student conditioned on a hint that names the teacher model's action, so the teacher only names actions \citep{zhao2026opsd,penaloza2026pid}. At each pivotal turn, \textbf{\emph{preventive distillation}} re-scores the student's recorded response with the self-teacher hinted with the gold action and minimizes a reverse KL loss, which shifts the student toward the gold action and away from the committed mistake. After each pivotal turn, \textbf{\emph{recovery distillation}} has the teacher model name a recovery action at each of the next few turns and trains the student, without the hint, on the responses that the hinted self-teacher writes toward them. Its mass-covering forward KL loss raises the probability of recovery actions that the student rarely samples \citep{kim2016sequence}.

Our contributions are as follows:
\begin{itemize}[leftmargin=3em, topsep=0pt, partopsep=0pt, itemsep=1pt, parsep=0pt]
\item \textbf{Diagnosis.} Using ALFWorld's oracle, we show that more than half of the failed rollouts contain a recoverable pivotal mistake, which standard OPD does not repair (\Cref{sec:evidence}).
\item \textbf{Method.} We propose \method{}, which uses a privileged self-teacher to train the student both to avoid pivotal mistakes and to recover from them (\Cref{sec:method}).
\item \textbf{Results.} Against $13$ baselines, \method{} achieves the best average performance on ALFWorld, WebShop, and Search-based QA for both Qwen3 students (\Cref{tab:main_results}), and at 8B it recovers from pivotal mistakes more than three times as often as standard OPD (\Cref{fig:case_study_curves}). With a Nemotron-3.5 student on SWE-Bench Verified, it also raises the resolve rate by $+3.2\%$, versus $+0.2\%$ for standard OPD (\Cref{fig:swebench}).
\end{itemize}

\section{Motivating Analysis: Agents Often Fail at a Pivotal Turn\texorpdfstring{\protect\\}{} and Rarely Recover from It}
\label{sec:evidence}

We ask whether failures hinge on a single decisive mistake, whether that mistake remains recoverable, and whether standard OPD repairs it (\Cref{fig:motivation}). On ALFWorld \citep{shridhar2020alfworld}, a symbolic oracle computes the remaining optimal trajectory from the full environment state, and we call its next action the \emph{oracle action}. We roll out Qwen3-8B, \mbox{Qwen3-30B-A3B}, and \mbox{Qwen3-235B-A22B} \citep{yang2025qwen3} on $140$ held-out tasks, replay each with this oracle, and analyze Qwen3-8B below (\Cref{app:analysis}).

\textbf{A single pivotal turn often determines the outcome, yet it is recoverable.}
Across the three models, $59\%$ of the failed trajectories contain an action that lengthens the remaining optimal trajectory or makes the task unsolvable (\Cref{app:analysis}). We call such an action a \textbf{\emph{pivotal mistake}}, and the turn where it is committed a \textbf{\emph{pivotal turn}} (formalized in \Cref{sec:formulation}). In failed trajectories with a pivotal turn, the first one typically arrives early, at a median of turn $8
$--$12$ out of $30$ (\Cref{fig:motivation}a). All three models then waste an average of $18$--$21$ more turns without recovering from it. 

Among the $72$ failed Qwen3-8B trajectories with a pivotal turn, correcting the first pivotal turn with the oracle action raises the replayed success from $8\%$ to $59\%$, whereas correcting a later turn helps far less (\Cref{fig:motivation}b). To test whether the post-mistake state is still recoverable by some admissible action sequence, we leave the mistake in place and force the oracle action at the next two turns, which still reaches $58\%$. Learning to recover after a pivotal mistake can therefore be nearly as effective as preventing it.

\textbf{Standard on-policy distillation does not repair pivotal turns.}
We train Qwen3-8B for $120$ steps with standard OPD, in which a privileged copy of the model provides token-level supervision on the model's own responses \citep{zhao2026opsd}. At each checkpoint, we use the same oracle to categorize the remaining failures (\Cref{fig:motivation}c). On held-out tasks, OPD lowers the overall failure rate by $23.6\%$ but the rate of failures after a pivotal turn by only $2.1\%$. Most of the gain comes from failures without a pivotal turn, which fall by $21.4\%$. OPD does suppress the committed mistake, moving its median probability from $0.999$ to below $10^{-5}$, yet the oracle action stays below $10^{-2}$ at every pivotal turn (\Cref{app:analysis}). A group of eight rollouts is therefore not expected to sample it even once, so it receives little direct reinforcement.

\section{\method{}: Pivot-Aware On-Policy Distillation}
\label{sec:method}

\looseness-1 \Cref{sec:evidence} shows that more than half of the failed rollouts contain a recoverable pivotal mistake that standard OPD does not repair. \method{} therefore concentrates the dense token-level signal of on-policy distillation at pivotal turns and the turns that follow them. It adds three components to group-based RL (\Cref{fig:method}): \emph{pivot detection} (\Cref{sec:pivotal_analysis}), where a teacher model names the actions that the student should take at and after pivotal turns, and \emph{preventive} and \emph{recovery distillation} (\Cref{sec:recovery}), where a privileged self-teacher provides the dense token-level targets. \Cref{sec:theory} analyzes the learning signal of recovery distillation, and \Cref{alg:method} in \Cref{app:method_details} summarizes one training step.

\subsection{Preliminaries}
\label{sec:formulation}

We model an agentic task as a partially observable Markov decision process \citep{kaelbling1998planning}. At turn $t$, the environment is in a latent state $s_t$ and emits an observation $o_t$, where $o_0$ states the task instruction. The agent holds a context $c_t = (o_0, y_0, o_1, y_1, \dots, o_t)$, which is the history of observations and responses up to the current observation. The current observation specifies the admissible actions $\mathcal{A}(c_t)$, which ALFWorld and WebShop list explicitly, Search-based QA defines as any search query or answer, and SWE-Bench defines as any call to the agent's tools or the final patch submission. The student policy generates the next response $y_t \sim \pi_\theta(\cdot \mid c_t)$, which includes free-form reasoning followed by a committed action $a_t := \act(y_t) \in \mathcal{A}(c_t)$. A trajectory $\tau = \{(o_t, y_t)\}_{t=0}^{T-1}$ records one attempt of $T$ turns and receives an outcome score $R(\tau)$ when it ends.

Following group-based RL for agents \citep{shao2024deepseekmath, feng2025gigpo}, we roll out each task $G$ times from the same initial state. The outcome scores of these rollouts yield a group-relative advantage $A^{\mathrm{RL}}$, which we optimize with the clipped PPO loss $\mathcal{L}^{\mathrm{RL}}$ \citep{schulman2017ppo}. We write $\pib$ for the student policy frozen at the start of each training step. A teacher model $M$ assists training, and in the self-distillation setting, $M$ is the student itself.

\textbf{Pivotal mistake \& pivotal turn.}
We now formalize the pivotal mistakes of \Cref{sec:evidence}. We write $L(s_t)$ for the length of the remaining optimal trajectory, which is the minimum number of turns needed to complete the task from the latent state $s_t$, with $L(s_t) = \infty$ once the task is unsolvable. A turn $t$ is a \emph{pivotal turn} if its committed action $a_t$ increases this length, i.e., $L(s_{t+1}) > L(s_t)$, and $a_t$ is then a \emph{pivotal mistake}. The pivotal turns of a trajectory $\tau$ form the set $\{t : L(s_{t+1}) > L(s_t)\}$, which can contain several turns even when $\tau$ succeeds. For each failed trajectory, \Cref{sec:evidence} uses the smallest element of this set. Pivot detection estimates these turns without access to $L$.

\subsection{Pivot Detection with a Teacher Model}
\label{sec:pivotal_analysis}

Computing $L$ requires an oracle that knows the optimal trajectory from every state (\Cref{sec:formulation}), such as the symbolic oracle of ALFWorld that \Cref{sec:evidence} uses. Since most environments offer no such oracle, \method{} instead detects pivotal turns during training with the teacher model (\Cref{fig:method}).

\textbf{Pivotal turns and gold actions.}
After collecting the rollouts of a training step, we prompt the teacher model to read each trajectory $\tau$ together with its outcome $R(\tau)$ and to select up to $m$ \emph{candidate turns} where the student may have gone wrong. At each candidate turn $t$, the teacher model also names a \emph{\textbf{gold action}} $a^{*}_t \in \mathcal{A}(c_t)$, which serves as its estimate of the oracle action. A candidate turn does not necessarily contain a mistake, since the student's committed action $a_t$ may already agree with the gold action. We therefore treat a candidate turn as pivotal only when the two disagree, i.e., $a_t \neq a^{*}_t$. The teacher-detected pivotal turns of a training step form the set $\mathcal{T}_\text{pivot}$. 
% Although ALFWorld provides an oracle, we use it only to evaluate pivot detection, so that training is identical across benchmarks.
At least one teacher-detected pivotal turn falls within one turn of the oracle-labeled pivotal turn in $77.8\%$ of failed ALFWorld trajectories on average over our two teachers, at least twice the rate of randomly chosen turns (\Cref{tab:teacher_oracle}). \Cref{app:method_details} compares pivot detection with this oracle labeling.

\textbf{Recovery actions.}
A gold action shows how to avoid a pivotal mistake, but the student also needs to learn how to continue from the state that the mistake creates. After each pivotal turn $t \in \mathcal{T}_\text{pivot}$, the teacher model names \emph{\textbf{recovery actions}} for up to $K$ \emph{recovery turns}, where $K$ is the recovery budget. We write $\tilde{c}_{t+k}$ for the context of the $k$-th recovery turn, $k = 1, \dots, K$. The first recovery turn starts from the post-mistake state, whose context is the recorded $\tilde{c}_{t+1} = c_{t+1}$, and \Cref{sec:recovery} describes how later recovery turns are reached. At each recovery turn, we query the teacher model again with $\tilde{c}_{t+k}$, and it names a recovery action $a^{*}_{t+k} \in \mathcal{A}(\tilde{c}_{t+k})$ as the next action toward completing the task.

\subsection{Preventive and Recovery Distillation}
\label{sec:recovery}

\textbf{Privileged self-teacher.}
To turn a gold or recovery action into a token-level target, we provide it to the frozen student $\pib$ as a \emph{hint}. For an action $a$, the hint $h(a)$ is a short instruction that presents $a$ as a reasonable action at the current turn and asks the model to reason toward it in its own words. Conditioning $\pib$ on this hint gives a \emph{privileged self-teacher} $\pib(\cdot \mid c, h(a))$, which differs from the student only through the information that the action provides \citep{zhao2026opsd, penaloza2026pid}. This keeps the distillation target in the student's own reasoning style.

\textbf{Preventive distillation.}
At each pivotal turn $t \in \mathcal{T}_\text{pivot}$, preventive distillation trains the student toward the self-teacher hinted with the gold action through the reverse KL loss
\begin{equation}
\mathcal{L}^{\mathrm{prev}}_{t}(\theta) = \KL\!\left( \pi_\theta\!\left(\cdot \mid c_t\right) \,\middle\|\, \pib\!\left(\cdot \mid c_t,\, h(a^{*}_t)\right) \right).
\label{eq:hint}
\end{equation}
This loss takes its expectation under the student, so we evaluate it on the student's recorded response $y_t$. Minimizing it shifts the student toward the gold action and away from the committed mistake.

\textbf{Recovery distillation.}
A reverse KL loss like \Cref{eq:hint} can only reweight responses that the student has already produced, while recovery distillation trains the student on responses that the self-teacher writes. At the $k$-th recovery turn after a pivotal turn $t$, we sample a recovery response $y_\text{rec} \sim \pib(\cdot \mid \tilde{c}_{t+k}, h(a^{*}_{t+k}))$ from the self-teacher hinted with the recovery action. We keep the response only if it commits to an action and does not refer to the hint (\Cref{app:method_details}). For $k < K$, we then execute its action $\act(y_\text{rec})$ in a copy of the environment that replays all preceding actions, which reaches the next recovery context $\tilde{c}_{t+k+1}$.

At each recovery turn, recovery distillation then trains the student toward this self-teacher through the forward KL loss
\begin{equation}
\mathcal{L}^{\mathrm{rec}}_{t,k}(\theta) = \KL\!\left( \pib\!\left(\cdot \mid \tilde{c}_{t+k},\, h(a^{*}_{t+k})\right) \,\middle\|\, \pi_\theta\!\left(\cdot \mid \tilde{c}_{t+k}\right) \right),
\label{eq:recovery}
\end{equation}
in which the student sees $\tilde{c}_{t+k}$ without the hint. This loss takes its expectation under the self-teacher, so we evaluate it on the accepted responses $y_\text{rec}$. Because the forward KL is mass-covering, minimizing it raises the student's probability of the recovery action at states that arise from its own mistakes (\Cref{sec:theory}). We select $K$ on validation and study its effect in \Cref{sec:discussion}.

\textbf{\method{} training objective.}
We combine group-based RL with preventive and recovery distillation in a single training objective,
\begin{equation}
\mathcal{L}(\theta) \;=\; \mathcal{L}^{\mathrm{RL}}(\theta)
\;+\; w_\text{prev} \sum_{t \in \mathcal{T}_\text{pivot}} \mathcal{L}^{\mathrm{prev}}_{t}(\theta)
\;+\; w_\text{rec} \sum_{t \in \mathcal{T}_\text{pivot}} \sum_{k=1}^{K} \mathcal{L}^{\mathrm{rec}}_{t,k}(\theta),
\label{eq:objective}
\end{equation}
where the sums cover the pivotal turns of the current training step and the recovery turns after each. We implement all three terms in a single PPO update over the rollout and recovery responses. The teacher model also writes brief feedback on each trajectory as a whole, which we distill at every turn with a small weight (\Cref{app:method_details}).

To implement both distillation losses within the PPO update, we assign the $\ell$-th token of a response $y$ at a context $c$ with a named action $a$ the \emph{distillation advantage}
\begin{equation}
A^{\mathrm{distill}}_\ell = \log \pib\!\left(y_\ell \mid c,\, h(a),\, y_{<\ell}\right) - \log \pib\!\left(y_\ell \mid c,\, y_{<\ell}\right),
\label{eq:distill_adv}
\end{equation}
which measures how much the hint $h(a)$ changes the frozen student's log-probability of this token. Preventive distillation adds $w_\text{prev}\, A^{\mathrm{distill}}_\ell$ to $A^{\mathrm{RL}}$ on each token of $y_t$, with $c = c_t$ and $a = a^{*}_t$, which implements a per-token form of \Cref{eq:hint}. Recovery distillation uses a clipped $w_\text{rec}\, A^{\mathrm{distill}}_\ell$ as the only advantage on each token of $y_\text{rec}$, with $c = \tilde{c}_{t+k}$ and $a = a^{*}_{t+k}$ (\Cref{app:method_details}). This update is not the gradient of \Cref{eq:recovery}, but without clipping it vanishes only at its minimizer (\Cref{lem:unclipped}).

\subsection{Theoretical Analysis: Learning Signal After a Pivotal Mistake}
\label{sec:theory}

Group-based RL and reverse KL losses like \Cref{eq:hint} train on responses that the student produces, whereas recovery distillation trains on responses that the self-teacher writes. We analyze how this difference affects the learning signal on the recovery action, treating the committed action at a recovery turn as one categorical decision (proofs in \Cref{app:theory}).

\begin{proposition}[Learning signal after a pivotal mistake]
\label{prop:rescue}
At a recovery turn, let $p$ and $q$ be the action distributions of the frozen student and its privileged self-teacher, and let $g(a^{*})$ be the expected update on the student's logit of the recovery action $a^{*}$. Then:
\textbf{(i)} at the frozen student, the recovery loss satisfies $\mathcal{L}^{\mathrm{rec}}_{t,k} \ge q(a^{*}) \log\bigl(1/p(a^{*})\bigr) - \log 2$;
\textbf{(ii)} if actions are sampled from $p$ and weighted by any per-action signal $w$, as in group-based RL or a reverse KL loss like \Cref{eq:hint}, then $g(a^{*}) = p(a^{*})\bigl(w(a^{*}) - \mathbb{E}_{p}[w]\bigr)$;
\textbf{(iii)} if actions are sampled from $q$ and weighted by the distillation advantage clipped at a bound $\delta$, as in recovery distillation, and the hint raises the probability of $a^{*}$ by a factor of at least $e^{\delta}$, then $g(a^{*}) \ge \delta\bigl(q(a^{*}) - p(a^{*})\bigr) > 0$.
\end{proposition}

When the student assigns little probability to the recovery action, its divergence from the self-teacher in (i) is large, yet the student-sampled updates in (ii) become too small to reduce it. Recovery distillation instead provides the update in (iii), whose size depends on the self-teacher rather than the student, as we confirm empirically in \Cref{sec:discussion} (\Cref{fig:recovery_kl}).

\begin{table}[t]
\centering
\fontsize{8.5}{9.5}\selectfont
\setlength{\tabcolsep}{3.5pt}
\renewcommand{\arraystretch}{1.15}
{%
\begin{tabular}{l@{\hspace{7pt}} ccccccc cccccccc cc}
\toprule
\multirow{2}{*}{\textbf{Method}}
& \multicolumn{7}{c}{\textbf{ALFWorld}}
& \multicolumn{8}{c}{\textbf{Search-based QA}}
& \multicolumn{2}{c}{\textbf{WebShop}} \\
\cmidrule(lr){2-8} \cmidrule(lr){9-16} \cmidrule(lr){17-18}
& Pick & Look & Clean & Heat & Cool & Pick2 & \cellcolor{gray!15}\textbf{Avg.}
& NQ & Triv & Pop & Hotp & 2Wk & MuS & Bam & \cellcolor{gray!15}\textbf{Avg.}
& Score & \cellcolor{gray!15}\textbf{Succ.} \\
\midrule

\rowcolor{gray!15}\multicolumn{18}{l}{\textit{Qwen3-1.7B}} \\
% \textbf{Qwen3-1.7B}
Base Model
 & 6.8 & 60.2 & 0.0 & 0.0 & 3.6 & 4.1 & 12.4 & 16.8 & 50.8 & 37.9 & 26.9 & 21.0 & 10.0 & 15.6 & 25.6 & 45.0 & 4.6 \\
OPSD
 & 23.7 & 31.2 & 10.9 & 0.0 & 2.2 & 6.5 & 12.4 & 43.4 & 57.9 & 48.2 & 33.3 & 34.0 & 10.0 & 30.5 & 36.8 & 49.2 & 10.1 \\
GRPO
 & 74.0 & 44.1 & 33.3 & 41.9 & 30.4 & 32.5 & 42.7 & 42.4 & 58.3 & 50.2 & 38.2 & 38.5 & 9.1 & 27.4 & 37.7 & 74.5 & 57.0 \\
Skill-GRPO
 & \cellcolor{cellbest}\textbf{89.8} & 62.4 & 73.0 & 50.4 & \cellcolor{cellsecond}\underline{68.8} & 43.9 & 64.7 & 43.0 & 57.3 & 44.3 & 26.2 & 34.3 & 10.0 & 20.2 & 33.6 & 69.0 & 53.9 \\
AgentOPSD
 & 72.3 & 35.5 & 59.8 & 56.4 & 58.7 & 30.1 & 52.1 & \cellcolor{cellsecond}\underline{45.6} & 57.0 & 41.7 & 33.7 & 32.7 & 8.7 & 24.0 & 34.8 & 69.0 & 46.4 \\
Skill-SD
 & 79.1 & 58.1 & 54.0 & 57.3 & 50.0 & 56.1 & 59.1 & 39.2 & 53.1 & 47.2 & 32.4 & 34.0 & 12.0 & 28.3 & 35.2 & 77.2 & 58.6 \\
PivotRL
 & 45.8 & \cellcolor{cellsecond}\underline{75.3} & 26.4 & 15.4 & 18.8 & 19.5 & 33.5 & 35.6 & 45.3 & \cellcolor{cellbest}\textbf{63.8} & 31.4 & 29.8 & \cellcolor{cellbest}\textbf{21.4} & 21.5 & 35.5 & 66.6 & 28.8 \\
TurnOPD
 & 63.3 & \cellcolor{cellsecond}\underline{75.3} & 57.5 & 50.4 & 30.4 & 41.5 & 53.1 & 22.0 & 43.0 & 43.0 & \cellcolor{cellbest}\textbf{40.5} & \cellcolor{cellsecond}\underline{47.2} & 16.5 & 34.9 & 35.3 & 70.6 & 55.5 \\
StepOPSD
 & 72.3 & 60.2 & 73.0 & 21.4 & 42.0 & 48.0 & 52.8 & 42.1 & \cellcolor{cellsecond}\underline{59.5} & 48.5 & \cellcolor{cellbest}\textbf{40.5} & 35.0 & 9.1 & 28.0 & 37.5 & 72.8 & 60.9 \\
TCOD
 & 66.7 & \cellcolor{cellsecond}\underline{75.3} & 60.9 & 36.8 & 22.5 & 49.6 & 51.9 & 24.9 & 44.7 & 47.2 & 35.0 & 46.0 & \cellcolor{cellsecond}\underline{17.8} & 29.0 & 34.9 & 77.2 & 46.1 \\
SOD
 & 70.1 & \cellcolor{cellbest}\textbf{83.9} & 55.2 & 42.7 & 39.9 & \cellcolor{cellbest}\textbf{65.0} & 59.5 & 24.9 & 52.4 & 51.5 & \cellcolor{cellsecond}\underline{39.8} & \cellcolor{cellsecond}\underline{47.2} & 15.5 & \cellcolor{cellsecond}\underline{36.1} & 38.2 & 66.6 & 54.7 \\
RLSD
 & 72.3 & 67.7 & 73.0 & 50.4 & 61.6 & 39.0 & 60.7 & 42.1 & 57.3 & 51.5 & 38.2 & 40.1 & 12.0 & 29.0 & \cellcolor{cellsecond}\underline{38.6} & \cellcolor{cellsecond}\underline{83.2} & 62.5 \\
SDAR
 & 83.1 & 55.9 & \cellcolor{cellsecond}\underline{92.5} & \cellcolor{cellsecond}\underline{71.8} & 58.0 & 48.0 & \cellcolor{cellsecond}\underline{68.2} & 42.4 & 57.3 & 46.3 & 36.9 & 42.4 & 8.1 & 26.5 & 37.1 & 64.6 & 57.0 \\
OPID
 & 83.1 & 57.0 & 65.5 & \cellcolor{cellsecond}\underline{71.8} & 46.4 & 43.9 & 61.3 & 44.0 & 57.9 & 49.2 & 35.9 & 37.2 & 11.0 & 27.1 & 37.5 & 76.6 & \cellcolor{cellsecond}\underline{68.0} \\
\textbf{\method{}}
 & \cellcolor{cellsecond}\underline{87.6} & 55.9 & \cellcolor{cellbest}\textbf{93.7} & \cellcolor{cellbest}\textbf{72.6} & \cellcolor{cellbest}\textbf{73.9} & \cellcolor{cellsecond}\underline{58.5} & \cellcolor{cellbest}\textbf{73.7} & \cellcolor{cellbest}\textbf{48.5} & \cellcolor{cellbest}\textbf{64.1} & \cellcolor{cellsecond}\underline{57.0} & 39.2 & \cellcolor{cellbest}\textbf{49.2} & 15.9 & \cellcolor{cellbest}\textbf{37.4} & \cellcolor{cellbest}\textbf{44.5} & \cellcolor{cellbest}\textbf{84.4} & \cellcolor{cellbest}\textbf{76.6} \\
\midrule
\rowcolor{gray!15}\multicolumn{18}{l}{\textit{Qwen3-8B}} \\
% \textbf{Qwen3-8B}
Base Model
 & 31.1 & 39.8 & 23.0 & 0.0 & 8.0 & 26.0 & 21.3 & 13.6 & 37.5 & 13.6 & 17.8 & 25.2 & 3.9 & 17.4 & 18.4 & 27.5 & 7.0 \\
OPSD
 & 40.1 & 49.5 & 63.2 & 28.2 & 46.4 & 52.8 & 46.7 & 47.2 & 67.6 & \cellcolor{cellbest}\textbf{57.3} & 36.2 & 42.1 & 11.0 & 41.7 & 43.3 & 47.2 & 18.2 \\
GRPO
 & 44.6 & 59.1 & 11.5 & 21.4 & 31.2 & 43.9 & 35.3 & 35.0 & 63.1 & 44.3 & 32.0 & 41.1 & 9.1 & 37.7 & 37.5 & 83.5 & 75.0 \\
Skill-GRPO
 & 86.4 & 49.5 & 73.6 & 77.8 & 13.0 & 89.4 & 65.0 & 36.2 & 54.0 & 36.2 & 23.3 & 38.2 & 8.1 & 38.6 & 33.5 & 78.0 & 74.2 \\
AgentOPSD
 & 94.9 & 76.3 & 85.1 & 52.1 & 67.4 & 67.5 & 73.9 & 48.9 & 62.5 & 40.1 & 29.4 & 44.0 & 12.9 & 35.2 & 39.0 & 81.4 & 65.6 \\
Skill-SD
 & 89.8 & 91.4 & \cellcolor{cellbest}\textbf{100.0} & \cellcolor{cellsecond}\underline{88.9} & 76.8 & 84.6 & 88.6 & 43.0 & 65.7 & 43.0 & 32.0 & 46.0 & 11.0 & 40.2 & 40.1 & 81.7 & 72.7 \\
PivotRL
 & 85.3 & 74.2 & 42.0 & 29.9 & 30.4 & 61.8 & 53.9 & 28.5 & 52.4 & 48.2 & 36.9 & 39.5 & 17.2 & 43.9 & 38.1 & 65.9 & 53.0 \\
TurnOPD
 & 85.3 & 82.8 & 50.0 & 51.3 & 42.0 & 57.7 & 61.5 & 25.6 & 56.6 & 45.3 & 38.2 & 45.3 & 15.9 & 39.6 & 38.1 & 84.2 & 58.6 \\
StepOPSD
 & 93.2 & 82.8 & \cellcolor{cellsecond}\underline{94.8} & 55.6 & 60.1 & 89.4 & 79.3 & 42.1 & 67.6 & \cellcolor{cellsecond}\underline{55.3} & 33.3 & 46.0 & 11.0 & 44.2 & 42.8 & 80.7 & 72.7 \\
TCOD
 & 85.3 & 82.8 & 54.0 & 37.6 & 42.0 & 53.7 & 59.2 & 18.4 & 63.8 & 45.3 & \cellcolor{cellbest}\textbf{41.1} & 48.2 & \cellcolor{cellsecond}\underline{18.4} & \cellcolor{cellsecond}\underline{46.7} & 40.3 & 79.0 & 55.5 \\
SOD
 & 78.0 & \cellcolor{cellsecond}\underline{93.5} & 57.5 & 44.4 & 61.6 & 61.8 & 66.1 & 22.7 & 63.8 & 42.4 & 36.9 & 46.6 & \cellcolor{cellsecond}\underline{18.4} & 38.3 & 38.4 & 83.9 & 43.0 \\
RLSD
 & \cellcolor{cellbest}\textbf{100.0} & \cellcolor{cellbest}\textbf{100.0} & 84.5 & \cellcolor{cellbest}\textbf{94.0} & 76.8 & 89.4 & \cellcolor{cellsecond}\underline{90.8} & 36.2 & 58.3 & 42.1 & 23.0 & 34.0 & 7.1 & 38.3 & 34.1 & 84.6 & 78.9 \\
SDAR
 & \cellcolor{cellsecond}\underline{96.6} & 66.7 & 84.5 & 71.8 & \cellcolor{cellbest}\textbf{83.3} & \cellcolor{cellbest}\textbf{100.0} & 83.8 & 45.0 & \cellcolor{cellsecond}\underline{68.0} & 44.3 & 34.0 & \cellcolor{cellsecond}\underline{50.5} & 11.0 & 39.3 & 41.7 & \cellcolor{cellsecond}\underline{85.5} & \cellcolor{cellsecond}\underline{79.7} \\
OPID
 & \cellcolor{cellbest}\textbf{100.0} & 66.7 & \cellcolor{cellsecond}\underline{94.8} & 71.8 & \cellcolor{cellsecond}\underline{79.7} & 89.4 & 83.7 & \cellcolor{cellbest}\textbf{52.4} & 65.4 & 50.2 & \cellcolor{cellsecond}\underline{40.1} & \cellcolor{cellbest}\textbf{51.1} & 15.9 & 44.2 & \cellcolor{cellsecond}\underline{45.6} & 81.1 & 73.4 \\
\textbf{\method{}}
 & \cellcolor{cellbest}\textbf{100.0} & \cellcolor{cellbest}\textbf{100.0} & \cellcolor{cellbest}\textbf{100.0} & 85.5 & 76.8 & \cellcolor{cellsecond}\underline{95.9} & \cellcolor{cellbest}\textbf{93.0} & \cellcolor{cellsecond}\underline{50.2} & \cellcolor{cellbest}\textbf{72.5} & 52.1 & 39.2 & 48.9 & \cellcolor{cellbest}\textbf{19.1} & \cellcolor{cellbest}\textbf{49.8} & \cellcolor{cellbest}\textbf{47.4} & \cellcolor{cellbest}\textbf{88.2} & \cellcolor{cellbest}\textbf{81.9} \\
\bottomrule
\end{tabular}}
\vspace{5pt}
\caption{\small \textbf{Results on ALFWorld, Search-based QA, and WebShop with Qwen3-1.7B and Qwen3-8B students.} Results are averaged over three seeds, Avg. columns are unweighted means over task types or datasets, and every method with a teacher uses Qwen3-30B-A3B (1.7B) or Qwen3.5-122B-A10B (8B). QA columns abbreviate Natural Questions, TriviaQA, PopQA, HotpotQA, 2WikiMultiHopQA, MuSiQue, and Bamboogle. We highlight the \colorbox{cellbest}{\textbf{best}} and \colorbox{cellsecond}{\underline{second-best}} results.}
\label{tab:main_results}
\end{table}

\section{Experiments}
\label{sec:experiments}

\subsection{Experimental Setup}
\label{sec:setup}

\textbf{Benchmarks.} We evaluate \method{} on four multi-turn agentic benchmarks. \textbf{(1) ALFWorld} \citep{shridhar2020alfworld} is an embodied household environment with six task types, where the agent completes language-specified goals through textual actions. \textbf{(2) WebShop} \citep{yao2022webshop} is a simulated e-commerce website where the agent searches for and purchases a product that satisfies an instruction. \textbf{(3) Search-based QA} \citep{jin2025searchr1} requires the agent to answer questions from seven open-domain QA datasets \citep{kwiatkowski2019nq,joshi2017triviaqa,mallen2023popqa,yang2018hotpotqa,ho20202wiki,trivedi2022musique,press2023bamboogle} by calling a search engine. \textbf{(4) SWE-Bench Verified} \citep{jimenez2023swebench} requires the agent to resolve real GitHub issues by exploring and editing a Python repository, and we use it to test transfer to software engineering and to a different model family (\Cref{sec:main_results}). \Cref{app:benchmark_examples} shows an example episode from each of the first three benchmarks.

\textbf{Models \& training configurations.} We use Qwen3-1.7B and Qwen3-8B \citep{yang2025qwen3} as students, paired with a Qwen3-30B-A3B teacher and a Qwen3.5-122B-A10B teacher, respectively. Every method that uses a teacher, including the baselines, uses the same teacher for a given student. In the self-distillation setting, the student serves as its own teacher. On the first three benchmarks, all methods share the same training data, $160$ training steps, and $8$ rollouts per task (\Cref{app:hyperparameters,app:prompts}).

\looseness-1 \textbf{Baselines.} Besides the base model, we compare \method{} against three groups of methods, namely \textbf{(1)} RL and self-distillation, with GRPO \citep{shao2024deepseekmath}, OPSD \citep{zhao2026opsd}, RLSD \citep{yang2026rlsd}, and SDAR \citep{lu2026sdar}; \textbf{(2)} turn-level distillation for multi-turn agents, with TurnOPD \citep{zhou2026turnopd}, TCOD \citep{wang2026tcod}, SOD \citep{zhong2026sod}, StepOPSD \citep{zhang2026stepopsd}, and AgentOPSD \citep{wang2026agentopsd}; and \textbf{(3)} high-level guidance through skills or pivotal turns, with Skill-GRPO and OPID \citep{yang2026opid}, Skill-SD \citep{wang2026skillsd}, and PivotRL \citep{yi2026pivotrl} (\Cref{app:baseline_details}).

\textbf{Evaluation.} We report the task success rate on each ALFWorld task type, the exact-match accuracy on each QA dataset, and both the normalized score and the success rate on WebShop. The ALFWorld and Search-based QA averages are unweighted means over task types and datasets, respectively. Each experiment is run with three random seeds, and we report the average. We select each checkpoint on a separate validation set and evaluate it on held-out test sets of $274$ ALFWorld tasks, $500$ WebShop instructions, and $725$ QA questions, sampling at temperature $0.4$ (\Cref{app:eval_details}).

\subsection{Main Results}
\label{sec:main_results}
\textbf{\method{} attains the best average performance on all three main benchmarks, especially where successful rollouts are scarce.} As \Cref{tab:main_results} shows, \method{} ranks first on all eight per-benchmark averages. With the 1.7B student, it improves over the strongest baseline by $+5.5\%$ on ALFWorld and $+5.9\%$ on Search-based QA. With the 8B student, the margins are smaller but remain at least $+1.8\%$. The comparison with GRPO, which learns from outcome rewards alone, shows that \method{} helps most where successful rollouts are scarce. With the 1.7B student, it improves over GRPO the most on Clean, Cool, and Heat, which are the three ALFWorld task types that the base model solves least often. When every rollout of a task fails, group-relative advantages provide no learning signal (\Cref{prop:starvation} in \Cref{app:theory}), whereas recovery distillation still trains the recovery action at the states where the student gets stuck (\Cref{sec:theory}).

\begin{figure}[t]
\centering
\includegraphics[width=\linewidth]{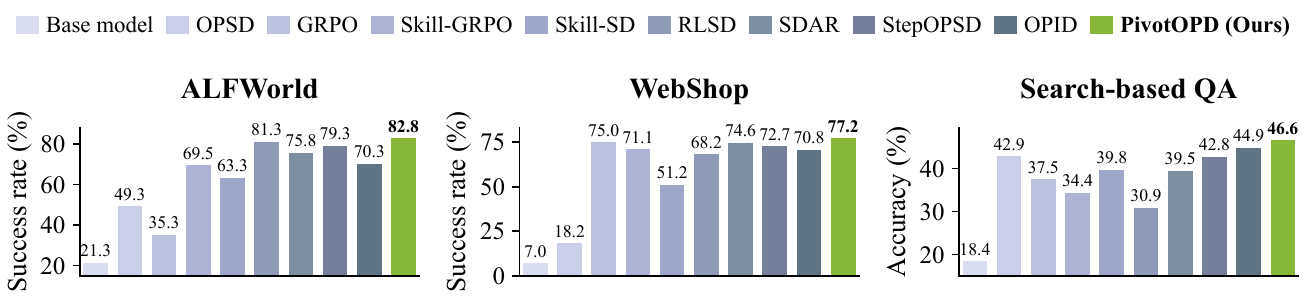}
\caption{\small \textbf{Self-distillation results on Qwen3-8B}, where the student serves as its own teacher. \method{} is best on all three benchmarks, outperforming the strongest baseline by $+3.9\%$ on average.}
\label{fig:self_distillation}
\end{figure}

\looseness-1 \textbf{\method{} turns partial progress into successful task completion.}
On WebShop, the normalized score gives partial credit for partially satisfying the instruction, whereas the success rate counts only fully completed tasks. With the 1.7B student, RLSD attains the highest score among the baselines. \method{} improves over RLSD by only $+1.2\%$ in score but by $+14.1\%$ in success rate. With both students, its gains over GRPO are likewise larger in success rate than in score. The larger gains in success rate are consistent with the effect of recovery distillation, which teaches the student to recover from pivotal mistakes that would otherwise leave a partially completed task unfinished. Indeed, removing recovery distillation ($K = 0$) substantially lowers the best WebShop validation score of the 1.7B student (\Cref{fig:ablation_k}).

\textbf{\method{} remains effective without a stronger external teacher.}
In its main setting, \method{} uses a stronger teacher to name the gold and recovery actions. To test whether it still works without such a teacher, we use the student as its own teacher for \method{} and for every baseline that uses a teacher. As shown in \Cref{fig:self_distillation}, \method{} achieves the best performance on all three benchmarks, outperforming the strongest baseline on each by at least $+1.5\%$. Relative to training under the stronger teacher (\Cref{tab:main_results}), it stays within $5\%$ on Search-based QA and in WebShop success rate but falls $10.2\%$ short on ALFWorld. These results suggest that much of the gain comes from where and how the teacher intervenes rather than from teacher capacity alone.

\begin{wrapfigure}{r}{0.22\textwidth}
\vspace{-14pt}
\centering
\includegraphics[width=0.22\textwidth]{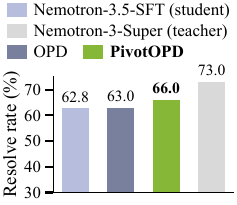}
\vspace{-12pt}
\captionsetup{justification=raggedright}
\caption{\textbf{SWE-Bench Verified results.} \method{} improves more than OPD.}
\label{fig:swebench}
\vspace{-7pt}
\end{wrapfigure}

\textbf{\method{} also improves a different model family on software engineering.}
The experiments so far all use Qwen3 students, so we next ask whether the gains of \method{} carry over to another model family and to software engineering. To this end, we train Nemotron-3.5-SFT with Nemotron-3-Super as the teacher, both from the Nemotron model family \citep{nvidia2025nemotron3}, and evaluate its resolve rate on SWE-Bench Verified (\Cref{app:swebench}). Since SWE-Bench episodes are long and containerized, we adapt pivot detection and the recovery budget to them (\Cref{app:swebench}). As shown in \Cref{fig:swebench}, \method{} raises the resolve rate by $+3.2\%$ and closes roughly a third of the gap to the teacher, whereas standard OPD improves it by only $+0.2\%$.

\section{Understanding Recovery from Pivotal Mistakes}
\label{sec:discussion}

We now study the recovery behavior behind these gains, asking \textbf{(1)} whether \method{} learns to recover from pivotal mistakes, \textbf{(2)} what makes this recovery learnable, \textbf{(3)} why recovery distillation provides a learning signal, and \textbf{(4)} how many recovery turns to train.

\begin{figure}[t]
\centering
\begin{subfigure}[b]{0.675\linewidth}
\centering
\includegraphics[width=\linewidth]{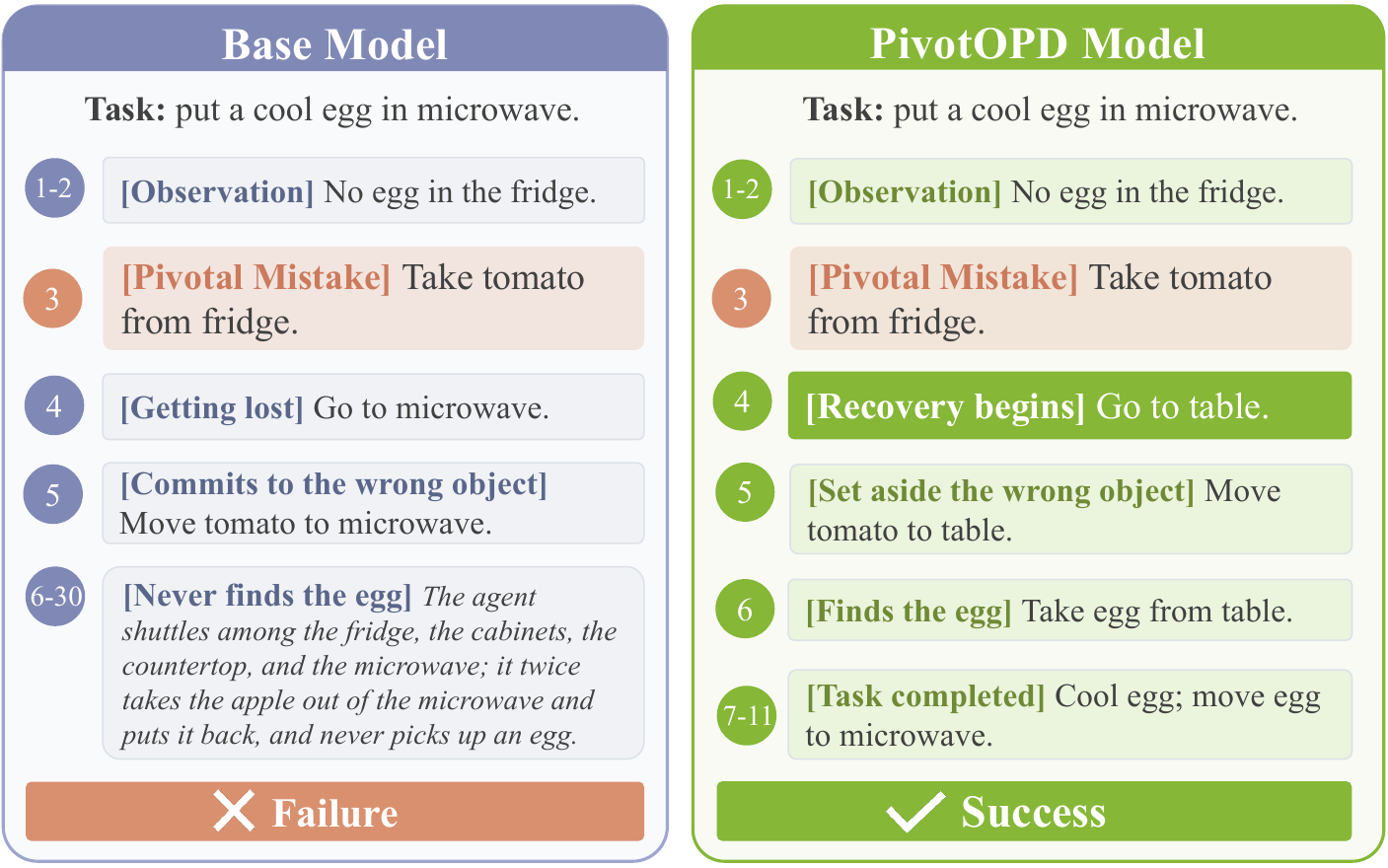}
\vspace{-10pt}
\caption{\small Case study: continuations after the same pivotal mistake.}
\label{fig:case_study_example}
\end{subfigure}
\hfill
\begin{subfigure}[b]{0.294\linewidth}
\centering
% trim order: left bottom right top (the export carries a 0.2in = 14.4pt pad on every side)
\includegraphics[width=\linewidth, trim={7pt 7pt 7pt 7pt}, clip]{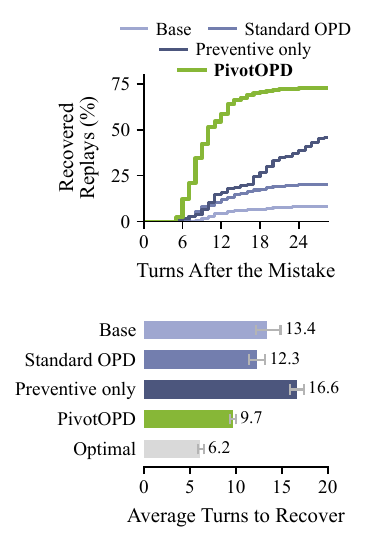}
\vspace{-10pt}
\caption{\small Recovery rate and efficiency.}
\label{fig:case_study_curves}
\end{subfigure}
\vspace{4pt}
\caption{\small \textbf{\method{} learns to recover from pivotal mistakes.} Each policy continues from a replayed prefix that ends with the pivotal mistake (\Cref{app:recovery_analysis}). \textbf{(a)} After the same pivotal mistake, the base model fails, whereas \method{} recovers and completes the task. \textbf{(b) Top:} the percentage of replays that recover within a given number of turns. \textbf{Bottom:} the average number of turns to recover among recovered replays, with the optimal number of turns in gray. \method{} recovers the most often and in the fewest turns.}
\label{fig:case_study}
\end{figure}

\textbf{\method{} learns to effectively recover from pivotal mistakes.}
We replay the $72$ ALFWorld pivotal mistakes of the motivating analysis (\Cref{sec:evidence}) and let each trained policy continue after the pivotal mistake (\Cref{app:recovery_analysis}). In the example of \Cref{fig:case_study_example}, after the student takes a tomato instead of the requested egg, the base model commits the tomato to the microwave and never finds the egg, whereas \method{} sets the tomato aside, takes the egg, and completes the task. Across all $72$ pivotal mistakes, \method{} recovers roughly nine times as often as the base model and in the fewest turns (\Cref{fig:case_study_curves}). Since the trained policy never sees the oracle, these recoveries also show that the mistakes are recoverable without its privileged information. It also raises the recovery rate by $+26.9\%$ over the preventive-only variant, i.e., \method{} with recovery budget $K = 0$, so explicit recovery distillation improves recovery beyond what preventive training alone achieves.

\begin{wraptable}{r}{0.35\textwidth}
\vspace{-6pt}
\centering
\small
\setlength{\tabcolsep}{3pt}
\renewcommand{\arraystretch}{1.15}
% Same 0.86 scale that the original seven-column table received from \resizebox.
\scalebox{0.87}{%
\begin{tabular}{l@{\hspace{3pt}} cccc}
\toprule
\textbf{Variant}
& {\footnotesize Clean} & {\footnotesize Heat} & {\footnotesize Cool} & {\footnotesize\textbf{Avg.}} \\
\midrule
Random pivotal turns
 & 76.2 & 61.1 & 50.0 & 62.6 \\
Hints w/o gold actions
 & 90.5 & 63.7 & \cellcolor{cellsecond}\underline{72.9} & 71.9 \\
Reverse-KL recovery
 & 85.7 & 50.0 & 61.5 & 64.5 \\
Recovery only
 & 66.7 & 40.9 & 53.9 & 64.7 \\
Preventive only
 & \cellcolor{cellsecond}\underline{93.2} & \cellcolor{cellsecond}\underline{70.2} & \cellcolor{cellsecond}\underline{72.9} & \cellcolor{cellsecond}\underline{72.5} \\
\textbf{\method{}}
 & \cellcolor{cellbest}\textbf{93.7} & \cellcolor{cellbest}\textbf{72.6} & \cellcolor{cellbest}\textbf{73.9} & \cellcolor{cellbest}\textbf{73.7} \\
\bottomrule
\end{tabular}}
\vspace{3pt}
\caption{\textbf{Component ablations on ALFWorld with the Qwen3-1.7B student.} Each variant changes one component of \method{} (preventive only also raises $w_\text{prev}$). The average is over all six task types (\Cref{tab:ablation_full}).}
\label{tab:ablation_components}
\vspace{-9pt}
\end{wraptable}

\looseness-1 \textbf{Recovery requires supervision at the right turns and on the right actions.}
To see what makes this recovery learnable, we train variants of \method{} on ALFWorld that each remove or replace a single component (\Cref{tab:ablation_components} and \Cref{app:additional_results}). Neither the preventive-only nor the recovery-only variant matches the average success rate of \method{}, so the two distillation terms are complementary. This holds even though the preventive-only variant compensates with a larger preventive weight (\Cref{app:additional_results}). The budget ablation shows a larger effect of recovery: the selected budget raises the best validation score over $K = 0$ by $5.4$ points on ALFWorld, $21.5$ on WebShop, and $2.4$ on Search-based QA (\Cref{fig:ablation_k}). When we inject the same hints at random turns, the average success rate is the lowest among all variants, indicating that supervision must land at pivotal turns. Similarly, when a generic reflection prompt replaces the gold action in each hint, the average success rate stays close to that of preventive only but drops sharply on Pick2 (\Cref{tab:ablation_full}), suggesting that the hint must name the correct action on task types whose decisions cannot be inferred from the observations alone.

\begin{figure}[t]
\centering
\includegraphics[width=\linewidth]{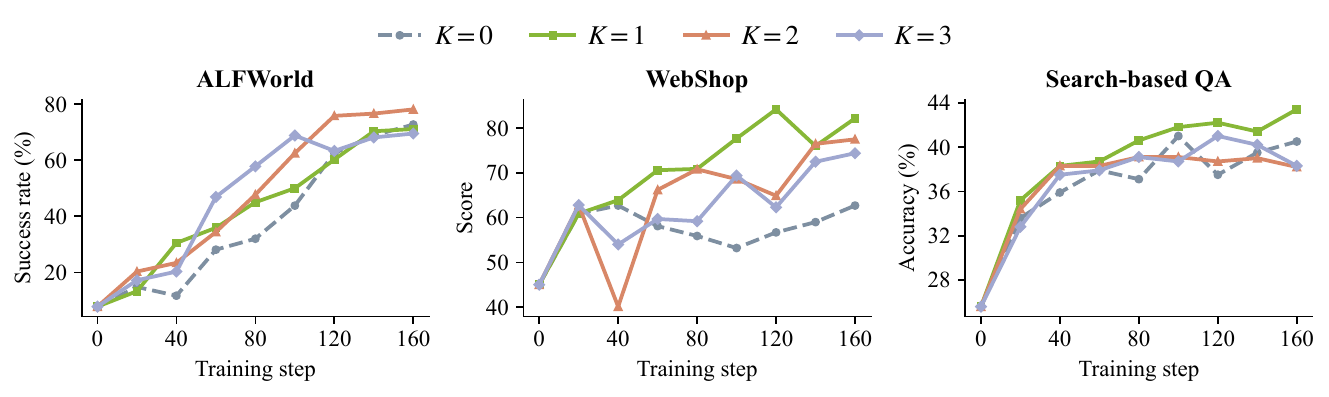}
\caption{\textbf{Ablation on the recovery budget with the Qwen3-1.7B student.} Each panel reports the validation performance trend with different numbers of recovery turns $K \in \{0, 1, 2, 3\}$ during training, where $K = 0$ is the preventive-only variant. Overall, $K=1$ yields the best performance for WebShop and Search-based QA, while ALFWorld needs $K=2$ recovery turns.}
\label{fig:ablation_k}
\end{figure}

\begin{wrapfigure}{r}{0.32\textwidth}
\vspace{-6pt}
\centering
\includegraphics[width=0.32\textwidth, trim={1pt 0pt 2pt 0pt}, clip]{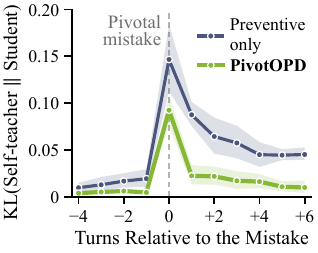}
\vspace{-7pt}
\caption{Recovery distillation quickly brings the KL divergence from the self-teacher back down after the pivotal mistake.}
\label{fig:recovery_kl}
\vspace{-4pt}
\end{wrapfigure}

\looseness-1 \textbf{Recovery distillation restores the learning signal that on-policy updates lose after a pivotal mistake.}
Beyond where and what to supervise, we now ask why recovery distillation provides a learning signal when the student rarely samples the recovery action. \Cref{prop:rescue} predicts that without recovery distillation, the student stays far from its privileged self-teacher after a pivotal mistake. In \Cref{fig:recovery_kl}, we track the forward KL divergence from the self-teacher to the student, which measures how far the student remains from its self-teacher (\Cref{app:recovery_kl}). After the mistake, the preventive-only student indeed remains at least twice as far from its self-teacher as \method{} at every displayed turn. This gap persists well beyond the single recovery turn that \method{} trains in this run ($K = 1$), suggesting that the learned recovery carries over to later turns rather than fitting the trained turn alone.

\textbf{The best recovery budget varies across benchmarks.}
It is $K = 1$ on WebShop and Search-based QA but $K = 2$ on ALFWorld (\Cref{fig:ablation_k}). One plausible explanation, consistent with \Cref{prop:depth} in \Cref{app:theory}, is that ALFWorld tasks require longer sequences of fine-grained actions after a mistake. \Cref{app:compute} reports the training compute overhead of each budget.

% \section{Related Works}
\section{Related Work}
\label{sec:related}
Our baselines span RL and on-policy self-distillation, alone or combined \citep{shao2024deepseekmath,zhao2026opsd,yang2026rlsd,lu2026sdar, he2026sdzero}, turn-level distillation for multi-turn agents \citep{zhou2026turnopd,wang2026tcod,zhong2026sod,zhang2026stepopsd,wang2026agentopsd}, and guidance from skills or pivotal turns \citep{wang2026skillsd,yang2026opid,yi2026pivotrl}, all of which learn from responses that the student samples. Several locate outcome-deciding turns, and OPID's turn-level skills are closest to our preventive distillation. \method{} additionally trains the student on self-teacher responses at the states that its own mistakes create, so it teaches recovery actions that the student rarely samples (\Cref{app:related_work}).

\section{Conclusion}
\label{sec:conclusion}
We showed on ALFWorld that failed rollouts often contain a recoverable pivotal mistake that standard OPD does not repair, and proposed \method{}, which distills a privileged self-teacher at and after pivotal turns. \method{} attains the best average performance against $13$ baselines on three agentic benchmarks and recovers from pivotal mistakes far more often than standard OPD. It also improves a different model family on software engineering.

\newpage
\appendix
\onecolumn

\clearpage

\renewcommand{\thesection}{\Alph{section}}
\renewcommand\thefigure{S.\arabic{figure}}
\setcounter{figure}{0}
\renewcommand\thetable{S.\arabic{table}}
\setcounter{table}{0}
\crefalias{section}{appendix}
\crefalias{subsection}{appendix}

\raggedbottom
\section{Motivating Analysis Details}
\label{app:analysis}

This appendix gives the protocol behind \Cref{fig:motivation} and \Cref{sec:evidence}.

\paragraph{Rollouts.} We roll out three models, Qwen3-8B, Qwen3-30B-A3B, and Qwen3-235B-A22B, on the $140$ held-out ALFWorld tasks of the \texttt{valid\_seen} split, sampling at temperature $0.4$ with a cap of $30$ turns and a history window of $5$ turns. The analysis runs on ALFWorld because it is the one benchmark of the three that provides a symbolic oracle, and the oracle is what makes pivotal turns identifiable without a model in the loop. \method{} does not use this oracle during training (\Cref{sec:pivotal_analysis}).

\paragraph{Oracle-labeled pivotal turns.} Every recorded trajectory is replayed in the environment alongside the ALFWorld PDDL oracle, which recomputes the number of turns that an optimal trajectory still needs, i.e., $L(s_t)$ in \Cref{sec:formulation}. A turn is pivotal when the action committed there raises this number or leaves the task unsolvable. Among the failed trajectories, $72$ of $111$ contain a pivotal turn for Qwen3-8B, $35$ of $70$ for Qwen3-30B-A3B, and $48$ of $81$ for Qwen3-235B-A22B, which amounts to $155$ of $262$ across the three models. The analysis uses the first pivotal turn of each failed trajectory, with the oracle action at that turn as the correction. Before this turn, the agent has not moved away from completing the task, so correcting it gives a clean counterfactual. The replay reproduces the recorded outcome for all $420$ trajectories, so the counterfactuals below are exact rather than approximate.

\paragraph{Counterfactual replay.} Qwen3-8B fails $111$ of the held-out tasks, and $72$ of these failed trajectories contain a pivotal turn. From the pivotal turn of each, we replay the remainder of the trajectory with the same model under the five interventions of \Cref{fig:motivation}b: no intervention, the model resampled at the pivotal turn, the oracle action forced at another turn drawn at random from the turns that are not pivotal, the oracle action forced at the pivotal turn itself, and the oracle action forced at each of the two turns after the mistake, with the mistake left in place and the oracle re-queried at every forced turn. Each intervention is repeated four to eight times per trajectory. The later-turn intervention is a control for the position of the correction, and it controls for position only when the drawn turn lands after the pivotal turn, since an earlier correction also erases the model's own mistake. We therefore report it over the $51$ trajectories where the drawn turn does land later, while the other interventions use all $72$. This control forces the oracle action later in the trajectory and has fewer turns left to work with, but success after correcting the pivotal turn is flat in the position of that turn, between $57\%$ and $63\%$ across position tertiles. Extending the correction after the mistake from two turns to three adds about four points, so the two-turn result is not an artifact of how many turns we correct.

\paragraph{Confidence intervals.} For each intervention, we average the replayed success over the replays of each trajectory and compute a $95\%$ bootstrap confidence interval over trajectories ($4{,}000$ draws, seed $0$), which we give in brackets in percent. Without intervention, the replayed success is $8.2\%$ $[3.8, 13.0]$, and resampling the pivotal turn gives $6.9\%$ $[2.4, 12.2]$. Forcing the oracle action at a later turn gives $17.2\%$ $[8.3, 27.5]$. Correcting the pivotal turn instead gives $59.0\%$ $[48.3, 69.8]$, and forcing the oracle action at the two turns after the mistake gives $58.3\%$ $[47.6, 69.1]$, or $62.2\%$ $[51.4, 72.9]$ with three forced turns. The intervals of correcting the pivotal turn and of forcing the oracle action after the mistake therefore lie entirely above those of the other three interventions.

\paragraph{On-policy distillation run.} Qwen3-8B is trained from its base checkpoint for $120$ steps with standard on-policy distillation. Here, standard OPD is the OPSD baseline of \Cref{tab:main_results}, which distills a privileged self-teacher \citep{zhao2026opsd} under the configuration of \Cref{app:hyperparameters,app:baseline_details}. For \Cref{fig:motivation}c, we roll out each saved checkpoint on the same $140$ held-out tasks, label the rollouts with the oracle exactly as above, and split each checkpoint's failures into those that fail at a pivotal turn and those whose trajectory contains no pivotal turn and leaves the task unfinished. Among the $72$ tasks that the base model loses at a pivotal turn, roughly seven in ten are lost again after distillation, and about half of the $72$ repeat the same kind of mistake. We score the probability that the model assigns to the oracle action under the frozen policy at each pivotal turn, before and after distillation. Distillation moves the median probability of the committed mistake from $0.999$ to below $10^{-5}$ and raises the median probability of the oracle action by ten orders of magnitude. The raised probability nevertheless stays below $10^{-2}$ at every pivotal turn, so a group of $8$ rollouts is not expected to sample the oracle action even once. This run is shorter than the $160$-step runs of \Cref{sec:setup}, and its numbers are not comparable to \Cref{tab:main_results}, though both sides of every comparison within this analysis are measured under the same protocol.

\section{Method Details}
\label{app:method_details}

This appendix details the components of \method{} summarized in \Cref{sec:method}, and \Cref{alg:method} presents one training step.

\paragraph{Teacher report and hints.} For each trajectory, the teacher returns one structured reply containing a trajectory summary, a trajectory-level lesson, a per-turn lesson, and gold-action blocks that name an action for up to $m$ candidate turns. The trajectory-level lesson is the brief feedback on the whole trajectory that \Cref{sec:recovery} mentions. It is inserted into the prompt at every turn of the trajectory and scored with the distillation advantage $A^{\mathrm{distill}}_\ell$, weighted by $w_\text{traj}$ in place of $w_\text{prev}$. For brevity, \Cref{eq:objective} omits this term. Replies that name no action contribute only the lesson signal and never trigger recovery. The hint $h(a)$ renders the action as a short passage. The passage states that $a$ is a reasonable approach at this turn and instructs the model to reason toward it from scratch, in its own style, without referencing, acknowledging, or quoting the hint. The exact templates are in \Cref{app:prompts}.

\paragraph{Action resolution.} The teacher's free-text action is matched onto the admissible set by normalized exact equality, then by equality after stripping articles, then by containment by a unique admissible command, and finally by highest token overlap above a fixed threshold, with a strict winner. If none of the rules resolves a reply, the recovery is dropped. For the open action spaces of Search-based QA and SWE-Bench, no matching is required: the teacher's named action is taken verbatim, subject to the same schema-validity check as student actions, i.e., a well-formed search query or answer for Search-based QA, and for SWE-Bench, a call that parses against the scaffold's tool interface and executes in the task container.

\paragraph{Mismatch check.} A candidate turn is pivotal when the student's committed action differs from the gold action after both are lowercased and their whitespace is collapsed. In Search-based QA, both actions are first rewritten in the forms \texttt{search[$\cdot$]} and \texttt{answer[$\cdot$]}. A response without a parsable action counts as a mismatch. Free-text actions such as search queries rarely match verbatim, so candidate turns with such actions are usually treated as pivotal. This is one source of the false positives discussed in the next paragraph.

\paragraph{Relation to oracle-labeled pivotal turns.} Pivot detection relaxes the oracle labeling of \Cref{sec:evidence} in four ways. \textbf{(1)} The teacher's judgment replaces the oracle. Although ALFWorld provides an oracle, we do not use it during training, so that \method{} stays identical across benchmarks and does not depend on a benchmark-specific oracle. \textbf{(2)} A trajectory may contain several pivotal turns, because the agent can move away from completing the task more than once. \textbf{(3)} Pivot detection covers successful trajectories as well as failed ones, because a successful trajectory can still contain a pivotal mistake from which the student happens to recover. \textbf{(4)} Disagreement with the gold action replaces an increase in $L$, so an equally reasonable alternative action can be flagged as pivotal. Such false positives are inexpensive. The preventive weight $w_\text{prev}$ is small (\Cref{tab:method_hparams}), and recovery distillation at such a turn still trains the student toward the teacher's next action at a state that the student actually visited.

\paragraph{Teacher--oracle agreement.} We measure how often pivot detection finds the oracle-labeled pivotal turns on ALFWorld. To this end, we run each training teacher on ALFWorld rollouts of the student that it trains, i.e., Qwen3-30B-A3B on Qwen3-1.7B and Qwen3.5-122B-A10B on Qwen3-8B (\Cref{sec:setup}). As in \Cref{app:analysis}, we keep the failed trajectories that contain an oracle-labeled pivotal turn. Each teacher uses the same prompt, parsing, and decoding configuration as in training (\Cref{app:prompts}). It therefore receives the full trajectory and its outcome, may choose among all turns, and selects up to $m = 5$ candidate turns. We draw $5$ samples per trajectory and score only the teacher-detected pivotal turns, i.e., the candidate turns whose gold action disagrees with the student's committed action. A sample is correct when at least one of these turns lies within one turn of the first oracle-labeled pivotal turn, and a reply that cannot be parsed counts as incorrect. We average correctness over the samples of each trajectory and then over trajectories. The random baseline selects as many turns as the sample's teacher-detected pivotal turns, uniformly at random without replacement, and applies the same criterion. Its value therefore differs between the two teachers, which are evaluated on different trajectories and detect different numbers of pivotal turns. Both teachers agree with the oracle at least twice as often as the random baseline (\Cref{tab:teacher_oracle}). In \Cref{sec:pivotal_analysis}, we say that at least one teacher-detected pivotal turn falls within one turn of the oracle-labeled pivotal turn when a sample is correct in this sense, and we report the mean accuracy of the two teachers, $77.8\%$.

\begin{table}[h]
\centering
\small
\setlength{\tabcolsep}{8pt}
\begin{tabular}{llccc}
\toprule
\textbf{Teacher} & \textbf{Student} & \textbf{\begin{tabular}[b]{@{}c@{}}$\pm1$-turn\\accuracy (\%)\end{tabular}} & \textbf{\begin{tabular}[b]{@{}c@{}}Random\\baseline (\%)\end{tabular}} & \textbf{\begin{tabular}[b]{@{}c@{}}Gain over\\random (\%)\end{tabular}} \\
\midrule
Qwen3-30B-A3B & Qwen3-1.7B & $71.1$ & $33.6$ & $+37.5$ \\
Qwen3.5-122B-A10B & Qwen3-8B & $84.4$ & $28.1$ & $+56.3$ \\
\bottomrule
\end{tabular}
\vspace{4pt}
\caption{\textbf{Teacher--oracle agreement on ALFWorld.} Each teacher analyzes failed rollouts of the student that it trains. Accuracy is the percentage of teacher samples with at least one teacher-detected pivotal turn within one turn of the first oracle-labeled pivotal turn, averaged over failed trajectories with such a turn. The random baseline selects the same number of turns uniformly at random, and the last column subtracts it from the accuracy.}
\label{tab:teacher_oracle}
\end{table}

\paragraph{Leakage control.} A generated recovery response is discarded if it matches any of a family of leakage patterns covering acknowledgments of a hint, suggestion, or provided approach. Training sequences are rebuilt from the unprivileged observation, and an assertion rejects any training prompt containing a privileged marker. As a result, hinted text can never enter the training prompt.

\paragraph{Environment replay.} Before taking recovery turns, multi-turn recovery replays the recorded action prefix in a pooled copy of the environment and verifies that the reached observation matches the recorded one exactly. A mismatch aborts the recovery. For $k \ge 2$, the recovery context $\tilde{c}_{t+k}$ is reached by executing $\act(y_\text{rec})$ of the previous recovery turn in this replayed environment rather than recorded in the original trajectory. The teacher is then queried again at $\tilde{c}_{t+k}$ for the next recovery action $a^{*}_{t+k}$. Each recovery turn adds one training sequence with its own loss $\mathcal{L}^{\mathrm{rec}}_{t,k}$. Beyond this check, at most $64$ recoveries are processed per training step, which bounds the time that recovery rollouts add to each step.

\paragraph{Combined advantage and surrogate.} The distillation terms of \Cref{eq:objective} enter the PPO update as per-token advantages built from the distillation advantage $A^{\mathrm{distill}}_\ell$ of \Cref{eq:distill_adv},
\begin{equation}
A^{\mathrm{prev}}_{\ell} = w_\text{prev}\, A^{\mathrm{distill}}_\ell,
\qquad
A^{\mathrm{rec}}_{\ell} = w_\text{rec}\, \operatorname{clip}_{[-\delta,\, \delta]}\!\left( A^{\mathrm{distill}}_\ell \right),
\label{eq:advantage}
\end{equation}
where $\ell$ indexes the tokens of the scored sequence and the clip bound $\delta$ caps the distillation advantage on any single token. For prevention, the distillation advantage is evaluated at the pivotal context $c_t$ with the gold action $a^{*}_t$ and applied to the recorded response $y_t$. Positive values increase the weight of tokens preferred by the hinted view, and negative values decrease the weight of tokens that it disfavors, so the update shifts the student's response toward the gold action. We keep the preventive weight $w_\text{prev}$ small (\Cref{tab:method_hparams}), since recovery distillation provides the main signal after a pivotal mistake. For recovery, the distillation advantage is evaluated at the recovery context $\tilde{c}_{t+k}$ with the recovery action $a^{*}_{t+k}$ and applied to $y_\text{rec}$. Unlike \Cref{eq:hint}, the recovery loss of \Cref{eq:recovery} takes its expectation under the frozen self-teacher, and gradients again flow only through $\pi_\theta$. The hint enters only this frozen target, while the trained response is conditioned on the unprivileged context $\tilde{c}_{t+k}$. Relative to standard on-policy self-distillation \citep{zhao2026opsd}, the privileged view is the hint that names the gold action, and the preventive term applies only at the pivotal turns rather than on every response. On a recovery sequence, the distillation advantage is near zero on the text that the plain policy would emit regardless and largest on the tokens that encode the recovery action, so the update concentrates where the hint mattered. The advantages combine with the group-relative advantage as
\begin{equation}
A_{\ell} \;=\;
\begin{cases}
A^{\mathrm{rec}}_{\ell} & \text{on recovery sequences}, \\[2pt]
A^{\mathrm{RL}}_{\ell} + A^{\mathrm{prev}}_{\ell} & \text{otherwise, where } A^{\mathrm{prev}}_{\ell} = 0 \text{ off pivotal turns}.
\end{cases}
\label{eq:combined}
\end{equation}
In both cases, $\ell$ indexes the tokens of the sequence. Writing $\rho_\ell(\theta) = \pi_\theta(y_\ell \mid c, y_{<\ell}) \,/\, \pib(y_\ell \mid c, y_{<\ell})$ for the token-level ratio at the context $c$ of each sequence, one PPO update minimizes the clipped surrogate
\begin{equation}
\widehat{\mathcal{L}}(\theta) \;=\; -\, \frac{1}{\sum_{(c,\, y)} |y|} \sum_{(c,\, y)} \sum_{\ell=1}^{|y|} \min\Bigl( \rho_\ell(\theta)\, A_{\ell},\; \operatorname{clip}_{[1-\epsilon_\text{clip},\; 1+\epsilon_\text{clip}]}\bigl(\rho_\ell(\theta)\bigr)\, A_{\ell} \Bigr),
\label{eq:loss}
\end{equation}
where the outer sum runs over every rollout and recovery sequence of the training step, $\ell$ indexes the tokens of each sequence, and $\epsilon_\text{clip}$ is the clip ratio. These advantages implement the corresponding terms of \Cref{eq:objective}, exactly in expectation for prevention and as a clipped, mass-covering step for recovery (\Cref{prop:starvation,prop:rescue}). The surrogate also carries a low-variance KL penalty (\Cref{app:hyperparameters}).

\begin{algorithm}[h]
\caption{One training step of \method{}}
\label{alg:method}
\begin{algorithmic}[1]
\Require frozen policy $\pib$, teacher $M$, group size $G$, weights $w_\text{prev}, w_\text{traj}, w_\text{rec}$, clip $\delta$, recovery turns $K$
\State collect $G$ rollouts per task with $\pib$; compute group-relative advantages $A^{\mathrm{RL}}$
\For{each trajectory $\tau$}
    \State $\{(t_j, a^{*}_{t_j})\}_{j \le m} \gets M(\tau)$ \Comment{candidate turns, gold actions, and lessons}
    \For{each candidate turn $t$ with gold action $a^{*}_t$}
        \If{$a_t \neq a^{*}_t$} \Comment{pivotal turn, $t \in \mathcal{T}_\text{pivot}$}
            \State add $A^{\mathrm{prev}}$ on the tokens of $y_t$ \Comment{\Cref{eq:advantage}}
            \State $c \gets c_{t+1}$
            \For{$k = 1, \dots, K$}
                \State $a^{*} \gets$ resolve$\big(M(c),\, \mathcal{A}(c)\big)$; \textbf{break} if unresolved
                \State $y_\text{rec} \sim \pib\!\left(\cdot \mid c,\, h(a^{*})\right)$; \textbf{break} if no action or hint leaked
                \State append $(c, y_\text{rec})$ with advantage $A^{\mathrm{rec}}$ and all other signals zeroed \Comment{\Cref{eq:advantage}}
                \If{$k < K$} \State $c \gets$ context after executing $\act(y_\text{rec})$ in the replayed environment \EndIf
            \EndFor
        \EndIf
    \EndFor
\EndFor
\State update $\theta$ on all sequences with the PPO clipped surrogate \Comment{\Cref{eq:loss}}
\end{algorithmic}
\end{algorithm}

\section{Analysis of Recovery Distillation}
\label{app:theory}

This appendix proves \Cref{prop:rescue} of \Cref{sec:theory}, states and proves two further propositions, and records a surrogate characterization of the recovery update. Parts (ii) and (iii) of \Cref{prop:rescue} follow the standard contrast between reverse and forward KL \citep{agarwal2024gkd,gu2024minillm}, and what is specific to \method{} is that recovery distillation applies the forward update at the states that the student's own mistakes create. We analyze the committed action at a single recovery turn as one categorical decision, which is the standard abstraction for token-level updates. Every statement holds verbatim per token, with the conditional next-token distributions in place of the action marginals and summed over the prefixes that the hinted policy visits. Fix a recovery turn with context $c = \tilde{c}_{t+k}$ and recovery action $a^{*} = a^{*}_{t+k}$, and write $p = \pib(\act(y) = \cdot \mid c)$ and $q = \pib(\act(y) = \cdot \mid c,\, h(a^{*}))$ over $\mathcal{A}(c)$. Write $u$ for the logits, so that $p_\theta = \softmax(u)$ has full support. At the start of an update, the PPO ratio equals one, so the surrogate gradient reduces to the advantage-weighted score function. All gradients are therefore evaluated at $p_\theta = p$.

The first result shows why preventive distillation and group-based RL provide little learning signal on the recovery action.

\begin{proposition}[Learning signal vanishes after a mistake]
\label{prop:starvation}
When the action-level distillation advantage $\log q(a) - \log p(a)$ is used as the per-action training signal on actions sampled from $p$, as in preventive distillation, the expected update is $-\nabla_u \KL(p_\theta \,\|\, q)$. The component of this update on an action $a$ is $p(a)\bigl(\log\tfrac{q(a)}{p(a)} + \KL(p \,\|\, q)\bigr)$, which vanishes on the recovery action as $p(a^{*}) \to 0$ at any fixed $q$. In a group of $G$ plain rollouts, the recovery action appears with probability $1 - (1 - p(a^{*}))^{G} \le G\, p(a^{*})$, and it first appears after $1 / p(a^{*})$ plain samples in expectation, against $1 / q(a^{*})$ under the hinted policy. Moreover, if every rollout of the task receives the same return, the group-relative advantage is identically zero.
\end{proposition}

\begin{proof}[Proof of \Cref{prop:starvation}]
Since $\mathbb{E}_{p_\theta}[\nabla_u \log p_\theta] = 0$,
$\nabla_u \KL(p_\theta \| q) = \mathbb{E}_{a \sim p}\!\left[ (\log p(a) - \log q(a)) \nabla_u \log p_\theta(a) \right]$,
which is the negative of the stated update. With $\partial \log p(b) / \partial u_a = \1[a = b] - p(a)$, the component on action $a$ is $p(a) \log \tfrac{q(a)}{p(a)} + p(a)\, \KL(p \| q)$, whose magnitude is at most $p(a) \bigl( |\log \tfrac{q(a)}{p(a)}| + \KL(p \| q) \bigr)$ and therefore vanishes as $p(a^{*}) \to 0$ at fixed $q$. For the sampling claims, the first appearance of $a^{*}$ under independent draws from $p$ is geometric with success probability $p(a^{*})$. This gives the mean $1/p(a^{*})$ and, by Bernoulli's inequality, $\Pr[a^{*} \text{ appears among } G \text{ draws}] = 1 - (1 - p(a^{*}))^{G} \le G\, p(a^{*})$. The hinted case is identical with $q$ in place of $p$. For the advantage claim, the group-relative advantage subtracts the group mean return, so identical returns give zero advantage on every token.
\end{proof}

For the clipped loss, the trainer-level analog of the identity in \Cref{prop:starvation} appears in prior work \citep{yang2026opid}. The component form in turn isolates what preventive distillation and group-based RL cannot do.

\begin{lemma}[Lower bound on the recovery loss]
\label{lem:kl_bound}
Let $\pi$ and $\pi'$ be two response distributions at the same context, and let $q$ and $p$ be their induced distributions over committed actions. Then for every action $a$ with $p(a) > 0$,
\begin{equation}
\KL(\pi \,\|\, \pi') \;\ge\; q(a) \log \frac{1}{p(a)} - \log 2.
\label{eq:kl_bound}
\end{equation}
In particular, with $\pi$ the privileged self-teacher and $\pi'$ the frozen student at $\tilde{c}_{t+k}$, the recovery loss of \Cref{eq:recovery} is at least $q(a^{*}) \log(1/p(a^{*})) - \log 2$.
\end{lemma}

\begin{proof}
The indicator $\1[\act(y) = a]$ is a function of the response $y$, so the data processing inequality gives $\KL(\pi \,\|\, \pi') \ge q(a) \log \frac{q(a)}{p(a)} + (1 - q(a)) \log \frac{1 - q(a)}{1 - p(a)}$. Since $1 - p(a) \le 1$, the second term is at least $(1 - q(a)) \log (1 - q(a))$. The right-hand side is therefore at least $q(a) \log \frac{1}{p(a)} - H(q(a))$, where $H(x) = -x \log x - (1 - x) \log (1 - x) \le \log 2$ is the binary entropy.
\end{proof}

\Cref{lem:kl_bound} holds for any hint, including the recorded hints of \Cref{app:recovery_kl}. It shows that the divergence from the self-teacher stays large whenever the self-teacher places substantial probability on an action that the student almost never takes, and that it can fall only when the student raises the probability of that action.

\begin{proof}[Proof of \Cref{prop:rescue}]
Part (i) is \Cref{lem:kl_bound} with $a = a^{*}$. For parts (ii) and (iii), let actions be sampled from a distribution $r$ and weighted by a signal $w$. Since $\partial \log p_\theta(b) / \partial u_a = \1[a = b] - p(a)$ at $p_\theta = p$, the expected update on $u_a$ is $\sum_b r(b)\, w(b) \bigl(\1[a = b] - p(a)\bigr) = r(a)\, w(a) - p(a)\, \mathbb{E}_r[w]$. Taking $r = p$ gives part (ii). For a bounded signal, $|w(a^{*}) - \mathbb{E}_p[w]| \le 2 \max_a |w(a)|$, so the expected update on $u_{a^{*}}$ is at most $2\, p(a^{*}) \max_a |w(a)|$ in magnitude. For part (iii), take $r = q$ and $w = \phi_\delta$, where $\phi_\delta(a) = \operatorname{clip}_{[-\delta,\delta]}\bigl(\log\tfrac{q(a)}{p(a)}\bigr)$ is the clipped distillation advantage. This gives
\begin{equation}
g_a \;=\; q(a)\,\phi_\delta(a) \;-\; p(a)\,\mathbb{E}_{q}\!\left[\phi_\delta\right].
\label{eq:recovery_coords}
\end{equation}
Under the condition $q(a^{*}) \ge e^{\delta}\, p(a^{*})$, we have $\phi_\delta(a^{*}) = \delta$. Since $\mathbb{E}_q[\phi_\delta] \le \delta$, \Cref{eq:recovery_coords} gives $g_{a^{*}} \ge \delta\, q(a^{*}) - \delta\, p(a^{*})$, which is positive because the condition implies $q(a^{*}) > p(a^{*})$.
\end{proof}

Without clipping, the recovery update on $a^{*}$ grows without bound as the student's probability of $a^{*}$ vanishes, and it stops only when the student's distribution matches the self-teacher's.

\begin{lemma}[Unclipped recovery update]
\label{lem:unclipped}
Suppose that $q(a^{*}) > p(a^{*})$. Without clipping, $g_{a^{*}} \to \infty$ as $p(a^{*}) \to 0$ while the remaining mass stays bounded away from zero. Moreover, the unclipped expected update of \Cref{eq:recovery_coords} vanishes in every component if and only if $p = q$.
\end{lemma}

\begin{proof}
Without clipping, \Cref{eq:recovery_coords} reads $g_{a^{*}} = q(a^{*}) \log \tfrac{q(a^{*})}{p(a^{*})} - p(a^{*})\, \KL(q \| p)$. Write $\KL(q \| p) = q(a^{*}) \log \tfrac{q(a^{*})}{p(a^{*})} + B$, where $B = \sum_{a \neq a^{*}} q(a) \log \tfrac{q(a)}{p(a)}$ stays bounded when the remaining mass is bounded away from zero. Then $g_{a^{*}} = (1 - p(a^{*}))\, q(a^{*}) \log \tfrac{q(a^{*})}{p(a^{*})} - p(a^{*})\, B \to \infty$. For the second claim, if $p = q$, both terms of \Cref{eq:recovery_coords} vanish. Conversely, suppose that $q(a) \log \tfrac{q(a)}{p(a)} = p(a)\, \kappa$ for all $a$, where $\kappa = \KL(q \| p) \ge 0$. Writing $r(a) = q(a)/p(a) > 0$, this reads $r(a) \log r(a) = \kappa$ for every $a$. If $\kappa = 0$, then $r \equiv 1$ and $p = q$. If $\kappa > 0$, then all $r(a)$ equal a single root $r^{*} > 1$, since $r \log r \le 0$ on $(0, 1]$ and is strictly increasing on $[1, \infty)$. But then $1 = \sum_a q(a) = r^{*} \sum_a p(a) = r^{*} > 1$, a contradiction.
\end{proof}

The last result concerns the recovery budget $K$.

\begin{proposition}[Deeper mistakes need more recovery turns]
\label{prop:depth}
Suppose that completing the task after the pivotal turn requires taking a specific action at each of $d$ successive turns, and that the plain policy places mass at most $\varepsilon < 1$ on each required action along this chain. After $K < d$ enforced recovery turns, the plain policy completes the remaining chain with probability at most $\varepsilon^{\,d - K}$.
\end{proposition}

\begin{proof}[Proof of \Cref{prop:depth}]
By the chain rule, the probability of taking all $d - K$ remaining required actions is the product of their conditional probabilities, each at most $\varepsilon$.
\end{proof}

\Cref{prop:depth} guides the recovery budget, since $K$ should grow with the depth $d$ of the required chain. Validation is consistent with this, selecting $K=2$ on ALFWorld, whose tasks chain subgoals over the longest trajectories, and $K=1$ on WebShop and Search-based QA (\Cref{fig:ablation_k}), although we do not measure the depth of individual mistakes. On ALFWorld, $K=2$ also matches the two guided turns after the pivotal mistake that restore success in the replay of \Cref{fig:motivation}b.

\begin{lemma}[Surrogate form of the recovery update]
\label{lem:surrogate}
Let $r = q / p$. The expected recovery update of \Cref{eq:recovery_coords} is the exact gradient of the frozen surrogate $F(u) = \mathbb{E}_{a \sim p_\theta}\!\left[ r(a)\, \phi_\delta(a) \right]$ at $p_\theta = p$, and without clipping, $F = \KL(q \,\|\, p)$ at $p_\theta = p$.
\end{lemma}
\begin{proof}
$\nabla_u F = \mathbb{E}_{p_\theta}[r\, \phi_\delta \nabla_u \log p_\theta]$, and at $p_\theta = p$, the reweighting by $r$ turns the expectation over $p$ into one over $q$, which yields the update of \Cref{prop:rescue}. Without clipping, at $p_\theta = p$, $F = \sum_a p(a) \tfrac{q(a)}{p(a)} \log \tfrac{q(a)}{p(a)} = \KL(q \| p)$.
\end{proof}

Three remarks connect the analysis to the implementation. First, the acceptance filters of \Cref{sec:recovery} replace the sampling distribution $q$ by its conditional on the accepted event while leaving the coefficients unchanged. The formulas above thus hold with $q(a)$ read as the accepted mass, and the expected update on the recovery action stays proportional to the mass that the filtered self-teacher places on it. Second, all statements concern the recovery context $\tilde{c}_{t+k}$, which follows the student's own mistake, and the student is trained without the hint. Supervised fine-tuning enforces teacher actions at the states that the teacher visits, whereas recovery distillation enforces the recovery action at the state produced by the student's own mistake with a single hinted generation. Recovery therefore avoids the compounding mismatch of imitation at the states that only an expert visits \citep{ross2011reduction}. Third, \Cref{prop:starvation,prop:rescue} make the division of labor in \method{} precise. Preventive distillation tempers the student where it already places mass, and recovery distillation plants mass exactly on the modes that the student has starved, at the turns that its mistakes produce.

\section{Experiment Details}
\label{app:exp_details}

This appendix complements the setup in \Cref{sec:setup} with the full training configurations (\Cref{app:hyperparameters}), the per-benchmark evaluation protocols (\Cref{app:eval_details}), and the baseline descriptions and configurations (\Cref{app:baseline_details}).

\subsection{Training Configurations}
\label{app:hyperparameters}

All methods share one on-policy training stack built on verl and train on H100 nodes. Training uses the standard ALFWorld training tasks, the WebShop instructions outside the held-out split, and the Search-R1 training split of Natural Questions and HotpotQA ($169{,}615$ questions). The student is served by vLLM for rollouts (tensor parallel $1$, thinking disabled) and updated with FSDP. Rollouts sample at temperature $1.0$ during training and $0.4$ at validation. Optimization uses token-mean loss aggregation, gradient clipping $1.0$, a low-variance KL loss, and discount $\gamma=0.95$. Teachers are served as vLLM endpoints with a $32$K context window, with tensor parallel $2$ for Qwen3-30B-A3B and $8$ for Qwen3.5-122B-A10B. Teacher calls use the served model's default sampling configuration. \Cref{tab:train_hparams} lists the per-benchmark hyperparameters, and \Cref{tab:method_hparams} lists the \method{} configuration. The recovery weight $w_\text{rec}$ is selected per benchmark with a halving sweep on validation performance, and on Search-based QA, the sweep selects a smaller weight for the 8B student. The recovery budget $K$ and the weight $w_\text{rec}$ interact, since each additional recovery turn adds one more distilled sequence per pivotal turn and thereby multiplies the total recovery signal. Deeper recovery therefore calls for a proportionally smaller weight, and on WebShop, $K=2$ requires a quarter of the $K=1$ weight to remain stable. Baselines receive the same tuning budget.

\begin{table}[h]
\centering
\small
\setlength{\tabcolsep}{6pt}
\renewcommand{\arraystretch}{1.05}
\begin{tabular}{l l l l}
\toprule
 & \textbf{ALFWorld} & \textbf{WebShop} & \textbf{Search-based QA} \\
\midrule
Tasks per step               & $16$              & $16$              & $64$ \\
Rollouts per task (group size) & $8$             & $8$               & $8$ \\
Trajectories per step        & $128$             & $128$             & $512$ \\
PPO mini-batch size          & $128$             & $64$              & $256$ \\
Learning rate                & $1\times10^{-6}$  & $1\times10^{-6}$  & $1\times10^{-6}$ \\
KL coefficient               & $0.01$            & $0.01$            & $0.01$ \\
Clip ratio                   & $0.2$             & $0.2$             & $0.2$ \\
Prompt / response length     & $2048$ / $512$    & $4096$ / $512$    & $4096$ / $512$ \\
Max turns per trajectory        & $30$              & $15$              & $4$ \\
History length               & $5$               & $2$               & $4$ \\
Training steps               & $160$             & $160$             & $160$ \\
Validation size / frequency  & $128$ / $10$      & $128$ / $10$      & $256$ / $10$ \\
\bottomrule
\end{tabular}
\vspace{4pt}
\caption{Training hyperparameters per benchmark. Both students share this configuration.}
\label{tab:train_hparams}
\end{table}

\begin{table}[h]
\centering
\small
\setlength{\tabcolsep}{6pt}
\renewcommand{\arraystretch}{1.05}
\begin{tabular}{l l l l}
\toprule
 & \textbf{ALFWorld} & \textbf{WebShop} & \textbf{Search-based QA} \\
\midrule
Candidate turns per trajectory $m$    & $5$    & $5$    & $2$ \\
Recovery turns $K$                    & $2$    & $1$    & $1$ \\
Recovery weight $w_\text{rec}$        & $1.0$  & $0.25$ & $0.0625$ (8B) / $0.125$ (1.7B) \\
Preventive / lesson weights $w_\text{prev}$ / $w_\text{traj}$ & $0.001$ / $0.001$ & $0.001$ / $0.001$ & $0.001$ / $0.001$ \\
Recovery clip $\delta$                & $5.0$  & $5.0$  & $5.0$ \\
Max recoveries per training step      & $64$   & $64$   & $64$ \\
\bottomrule
\end{tabular}
\vspace{4pt}
\caption{\method{} hyperparameters per benchmark. The preventive-only ablation uses $w_\text{prev} = 0.1$ (\Cref{app:additional_results}).}
\label{tab:method_hparams}
\end{table}

\subsection{Evaluation Details}
\label{app:eval_details}

For every method and benchmark, we train for $160$ steps, select the best checkpoint on a validation set, and evaluate the selected checkpoint on a held-out test set. Every experiment is run with three random seeds, and all reported results are averages over them. Validation and test rollouts sample at temperature $0.4$ and are capped at $30$ turns on ALFWorld, $15$ on WebShop, and $4$ on Search-based QA.

\paragraph{ALFWorld.} Checkpoint selection uses $128$ validation tasks evaluated every $10$ training steps at temperature $0.4$. The selected checkpoint is evaluated on all $274$ held-out tasks, i.e., the $140$ tasks of the \texttt{valid\_seen} split and the $134$ tasks of the \texttt{valid\_unseen} split. The reported average is the unweighted mean of the six task-type success rates.

\paragraph{WebShop.} The benchmark provides $6{,}910$ instructions over $1{,}000$ products. The first $500$ instructions are held out, and the rest are used for training. Checkpoint selection uses a $128$-instruction validation set evaluated every $10$ steps. The selected checkpoint runs on each of the $500$ held-out instructions. The success rate counts strictly completed purchases, and the score is the dense task score. Product retrieval uses BM25 ($k_1=1.5$, $b=0.75$) over the title, description, category, bullet-point, and option fields at both training and evaluation.

\paragraph{Search-based QA.} Checkpoint selection uses a mixed validation set of $256$ questions evaluated every $10$ steps. The per-dataset accuracies use a balanced evaluation set of $725$ questions, drawn with a fixed seed from a held-out pool of $51{,}713$ questions, with $100$ per dataset plus all $125$ Bamboogle questions. The metric is strict exact match, and the reported average is the unweighted mean of the seven per-dataset accuracies. The agent queries an E5 retriever \citep{wang2022e5} over the 2018 Wikipedia corpus and receives the top-$3$ passages per call, identically at training and evaluation.

\subsection{Baseline Descriptions and Configurations}
\label{app:baseline_details}

\paragraph{Standard post-training.}
\begin{itemize}
\item \textbf{GRPO} \citep{shao2024deepseekmath} optimizes a group-relative advantage from outcome rewards.
\item \textbf{OPSD} \citep{zhao2026opsd} distills a privileged self-teacher on the student's own responses.
\item \textbf{RLSD} \citep{yang2026rlsd} combines the two through teacher-guided advantage rescaling.
\item \textbf{SDAR} \citep{lu2026sdar} combines the two through a gated distillation loss.
\end{itemize}

\paragraph{Turn-level distillation for multi-turn agents.}
\begin{itemize}
\item \textbf{TurnOPD} \citep{zhou2026turnopd} makes the distillation signal turn-aware.
\item \textbf{TCOD} \citep{wang2026tcod} expands the distilled trajectory depth from short to long with a curriculum.
\item \textbf{SOD} \citep{zhong2026sod} reweights turn-level divergences.
\item \textbf{StepOPSD} \citep{zhang2026stepopsd} weights turns by hindsight from successful trajectories in the same rollout group.
\item \textbf{AgentOPSD} \citep{wang2026agentopsd} converts outcome rewards into turn-level credit through recursive self-distillation.
\end{itemize}

\paragraph{High-level guidance through skills or pivotal turns.}
\begin{itemize}
\item \textbf{Skill-GRPO}, the skill-augmented GRPO baseline of OPID \citep{yang2026opid}, injects retrieved skills into training prompts.
\item \textbf{Skill-SD} \citep{wang2026skillsd} conditions the self-teacher on retrieved skills.
\item \textbf{OPID} distills trajectory- and turn-level skills extracted from the student's own rollouts.
\item \textbf{PivotRL} \citep{yi2026pivotrl} concentrates updates on pivotal turns identified from rollout outcomes.
\end{itemize}

\paragraph{Configurations.} All baselines share the benchmarks, data, and optimization of \Cref{app:hyperparameters}, and GRPO uses the same advantage estimator with no teacher signal. The skill-based baselines (OPSD, Skill-SD, RLSD, SDAR, and Skill-GRPO) draw on a per-benchmark skill bank distilled from the student's own earlier rollouts, with $47$ entries for ALFWorld, $200$ for WebShop, and $75$ for Search-based QA. Skill-GRPO injects the top-$6$ retrieved skills into training prompts only. StepOPSD instead uses the first successful trajectory from the same rollout group as hindsight guidance. The auxiliary distillation coefficient is $1.0$ for OPSD, $0.001$ for Skill-SD, and $0.01$ for SDAR. RLSD and StepOPSD reshape advantages and carry no auxiliary loss.

\paragraph{Implementation differences.} Two implementation choices differ from the original papers. OPSD's full-vocabulary divergence is approximated with a sampled-token reverse-KL surrogate, and the RLSD and StepOPSD self-teachers re-synchronize every training step instead of every ten.

\subsection{Prompts}
\label{app:prompts}

This appendix lists the prompt templates of \method{}. Placeholders appear in shaded curly braces, tags that the model must emit are shown in green, and dashed lines separate the parts of each template. Several templates address the model directly, writing ``step'' for what the paper calls a turn.

\paragraph{Student turn prompt.} At every turn, the student receives one prompt containing the task, the current observation, and the admissible actions. It responds with reasoning in \texttt{<think>} tags followed by one action in \texttt{<action>} tags. The example below is the first turn of a held-out ALFWorld task, with the location and action lists abbreviated; WebShop and Search-based QA use the same structure with their own action formats.

\begin{promptbox}{Student turn prompt (ALFWorld)}
You are an expert agent operating in the ALFRED Embodied Environment.\\
Your current observation is: -= Welcome to TextWorld, ALFRED! =-\\[5pt]
You are in the middle of a room. Looking quickly around you, you see a bed 2, a bed 1, a desk 1, a drawer 11, ..., a drawer 1, a dresser 1, a garbagecan 1, a safe 1, a sidetable 2, and a sidetable 1.
\promptsep
Your task is to: look at alarmclock under the desklamp.\\
Your admissible actions of the current situation are: ['go to bed 1', 'go to bed 2', 'go to desk 1', 'go to drawer 1', ..., 'go to sidetable 2', 'inventory', 'look'].
\promptsep
Now it's your turn to take an action.\\
You should first reason step-by-step about the current situation. This reasoning process MUST be enclosed within \ptag{<think> </think>} tags.\\
Once you've finished your reasoning, you should choose an admissible action for current step and present it within \ptag{<action> </action>} tags.
\end{promptbox}

\paragraph{Teacher prompt.} For each trajectory, the teacher receives the task, the outcome, the indices of all turns as \emph{eligible turns}, and the formatted trajectory, and it returns the report of \Cref{app:method_details} in one reply. The template calls the eligible turns candidate step indices, and the turns that the teacher selects among them are the candidate turns of \Cref{sec:pivotal_analysis}. The JSON fields \texttt{episode\_summary}, \texttt{episode\_lesson}, and \texttt{step\_lessons} carry the trajectory summary, the trajectory-level lesson, and the per-turn lessons, and one \texttt{<correct\_action>} block per chosen turn names the gold action. The action-format example is benchmark-specific, shown here for ALFWorld. On WebShop and Search-based QA, the template appends short environment rules that force each recommended action to be a literal executable action, such as clicking every required option value before a purchase or adding new disambiguating terms to a search query.

\begin{promptbox}{Teacher prompt}
Analyze the following agent episode. Think deeply, step by step, about which step was the most critical for the outcome and what the CORRECT action would have been at that step.
\promptsep
You MUST output BOTH of the following:\\[5pt]
1) A single JSON object (no other JSON elsewhere) with EXACTLY these fields:
\begin{promptcard}
\{\\
\hspace*{2ex}"episode\_summary": "string",\\
\hspace*{2ex}"episode\_lesson": "string",\\
\hspace*{2ex}"step\_lessons": \{\\
\hspace*{4ex}"<idx>": "policy-facing imperative lesson text for step idx"\\
\hspace*{2ex}\}\\
\}
\end{promptcard}
Use up to \ph{max lesson count} keys in step\_lessons, chosen from the candidate indices below. Each step\_lessons value should be one short imperative sentence.\\[5pt]
2) For EACH step you placed in step\_lessons, ALSO emit a block of the exact form:
\begin{promptcard}
\ptag{<correct\_action step="N">}action text the agent should have emitted at step N\ptag{</correct\_action>}
\end{promptcard}
The action text must match the agent's action format (e.g. go to drawer 1). Do not include explanations or quotes inside the block.
\promptsep
\textbf{Important constraints:}\\
- Step indexing is 0-based.\\
- The chosen steps in step\_lessons MUST come from the candidate indices list below.\\
- Wrap the JSON in a single \textasciigrave\textasciigrave\textasciigrave json ...\textasciigrave\textasciigrave\textasciigrave{} fenced block.
\promptsep
\textbf{Episode context:}\\
- Task description: \ph{task description}\\
- episode\_success: \ph{outcome}\\
- Candidate step indices: \ph{eligible turn indices}\\
- Interaction trajectory: \ph{formatted trajectory}
\end{promptbox}

\paragraph{Recovery-action prompt.} At each recovery turn after a pivotal mistake, the teacher receives the student-facing observation of the current post-mistake state and names the single best next action, which becomes the recovery action of \Cref{sec:method}. The situation block already contains the task, the recent history, and the admissible actions, so the teacher sees exactly what the student sees.

\begin{promptbox}{Recovery-action prompt (ALFWorld)}
You are an expert alfworld agent. Below is the exact situation another agent is currently facing (its task, recent history, current observation, and the list of admissible actions).\\[5pt]
Decide the single best next action to make progress on the task from THIS state.
\promptsep
\textbf{Rules:}\\
- Choose the action verbatim from the admissible actions listed in the situation.\\
- Output exactly one block of the form:
\begin{promptcard}
\ptag{<correct\_action>}action text\ptag{</correct\_action>}
\end{promptcard}
- The action text must match the agent's action format (e.g. go to drawer 1).\\
- Do not add explanations inside the block. Keep any reasoning outside it brief.
\promptsep
\textbf{Situation:}\\
\ph{observation at the current post-mistake state}
\end{promptbox}

\paragraph{Hint.} The hint $h(a)$ renders a named action as a short passage that only the self-teacher sees. The same passage serves preventive distillation with the gold action and recovery distillation with the recovery action, and the leakage control of \Cref{app:method_details} discards any response that acknowledges it.

\begin{promptbox}{Hint $h(a)$}
A reasonable approach at this step is: \ph{named action}\\[3pt]
Reason from scratch in your own style as if you arrived at this approach independently. Do NOT reference, acknowledge, or quote the hint above in your output. Decide and act as your own choice.
\end{promptbox}

\subsection{Benchmark Examples}
\label{app:benchmark_examples}

We show one successful episode from each benchmark to illustrate its tasks, observations, and actions. Each episode is copied from a rollout log. For brevity, we show only the committed action of each turn, omit the reasoning and the list of admissible actions, and abbreviate long observations with [...]. Each numbered card is one turn, and highlighted text marks the evidence that determines the next action. The three episodes come from different policies, namely the Qwen3-8B student after $120$ steps of standard on-policy distillation on ALFWorld (\Cref{app:analysis}), Qwen3-235B-A22B on WebShop, and Qwen3-8B during AgentOPSD training on Search-based QA.

\paragraph{ALFWorld.} The agent must find a spatula, clean it at the sink, and place it on the dining table. It completes the task in seven turns.

\begin{examplebox}{Example episode (ALFWorld, Clean task type)}
\extask{Task}{clean some spatula and put it in diningtable.}
\exinit{You are in the middle of a room. Looking quickly around you, you see a cabinet 2, a cabinet 1, a coffeemachine 1, a countertop 3, [...], a diningtable 1, [...], a sinkbasin 1, [...], and a toaster 1.}
\exturn{1}{look}{You are in the middle of a room. Looking quickly around you, you see nothing.}
\exturn{2}{go to diningtable 1}{You arrive at diningtable 1. On the diningtable 1, you see a bread 2, a bread 1, [...], a \exhl{spatula 2}, and a spoon 2.}
\exturn{3}{take spatula 2 from diningtable 1}{You pick up the spatula 2 from the diningtable 1.}
\exturn{4}{go to sinkbasin 1}{You arrive at sinkbasin 1. On the sinkbasin 1, you see a cup 2, a fork 2, and a knife 1.}
\exturn{5}{clean spatula 2 with sinkbasin 1}{You \exhl{clean the spatula 2} using the sinkbasin 1.}
\exturn{6}{go to diningtable 1}{You arrive at diningtable 1. On the diningtable 1, you see a bread 2, a bread 1, [...], and a spoon 2.}
\exturn{7}{move spatula 2 to diningtable 1}{You move the spatula 2 to the diningtable 1.}
\exoutcome{task completed (reward 10)}
\end{examplebox}

\paragraph{WebShop.} The agent must buy a product that matches every attribute, option, and price constraint of the instruction. It searches once, opens a matching product, selects the requested color and size, and buys it.

\begin{examplebox}{Example episode (WebShop)}
\extask{Instruction}{Find me men's shorts with drawstring closure, elastic waist for gym workout with color: \#2 army green, and size: 30, and price lower than 30.00 dollars}
\exturn{1}{search[men's gym workout shorts drawstring closure elastic waist army green size 30 under \$30]}{`Back to Search' [SEP] `Page 1 (Total results: 50)' [SEP] `Next $>$' [SEP] \exhl{`B07WMMYB6G'} [SEP] `MAGCOMSEN Men's Gym Workout Shorts with Pockets Mesh Liner Quick Dry Running Shorts for Jogging, Hiking' [SEP] \exhl{`\$17.98 to \$17.99'} [SEP] `B092T8CY89' [SEP] [...]}
\exturn{2}{click[b07wmmyb6g]}{`Back to Search' [SEP] `$<$ Prev' [SEP] `color' [SEP] `\#1 black grey' [SEP] [...] [SEP] \exhl{`\#2 army green'} [SEP] [...] [SEP] `size' [SEP] \exhl{`30'} [SEP] `32' [SEP] [...] [SEP] `Price: \$17.98 to \$17.99' [SEP] [...] [SEP] `Buy Now'}
\exturn{3}{click[\#2 army green]}{(the same product page as in turn 2)}
\exturn{4}{click[30]}{(the same product page as in turn 2)}
\exturn{5}{click[buy now]}{Thank you for shopping with us! [SEP] [...] [SEP] Purchased [SEP] asin [SEP] B07WMMYB6G [SEP] options [SEP] \exhl{\texttt{\{"color": "\#2 army green", "size": "30"\}}} [SEP] [...] [SEP] Your score (min 0.0, max 1.0) [SEP] 1.0}
\exoutcome{purchased a product that matches the instruction (score 1.0)}
\end{examplebox}

\paragraph{Search-based QA.} The agent must answer a two-hop question. It first searches for the director of the film and then for the TV drama that this director co-created.

\begin{examplebox}{Example episode (Search-based QA)}
\extask{Question}{What is the TV drama co-created by this Irish film and TV writer/director who directed the 2001 film ``On the Edge''?}
\exturn[Information]{1}{<search>Director of the 2001 film "On the Edge"</search>}{Doc 1: ``On the Edge (2001 film)'' On the Edge (2001 film) On the Edge is a 2001 Irish film \exhl{directed by John Carney} and starring Cillian Murphy, Tricia Vessey, Jonathan Jackson and Stephen Rea. [...]}
\exturn[Information]{2}{<search>TV drama co-created by John Carney</search>}{Doc 1: ``John Carney (director)'' John Carney (director) John Carney (born 1972) is an Irish film and TV writer/director who specialises in low-budget indie films. [...] He is also a co-creator of the Irish TV drama series ``\exhl{Bachelors Walk}''. [...]}
\exturn{3}{<answer>Bachelors Walk</answer>}{}
\exoutcome{exact match with the gold answer (score 1.0)}
\end{examplebox}

\subsection{SWE-Bench Verified}
\label{app:swebench}
\paragraph{Training.} We train on a curated software-engineering curriculum of root-cause bug-fix instances augmented with single-file SWE-rebench tasks, decontaminated against SWE-Bench Verified by removing all overlapping repositories and issues. The agent operates an OpenHands-style function-calling scaffold whose tools are repository exploration, file editing, shell execution, and patch submission, with a $100$-turn and $3{,}600$-second budget per episode and a $196$K-token context. Each training step rolls out $G = 16$ trajectories for each of $32$ instances at temperature $0.6$, and the teacher is served as a vLLM endpoint. Because SWE-Bench episodes are long and containerized, we adapt the configuration as follows. For each task group with failures, the teacher audits one failed trajectory at its final committed action ($m = 1$) and names a gold action. The hint uses the same template as in \Cref{app:prompts}. Unlike on the other benchmarks, the hinted distribution comes from Nemotron-3-Super rather than from a privileged self-teacher, and it enters the unclipped preventive advantage $A^{\mathrm{prev}}_\ell$ of \Cref{eq:advantage} at weight $w_\text{prev} = 0.1$. The recovery budget is $K = 0$, since the audited turn is the final committed action, after which the recorded trajectory contains no turn at which to apply recovery distillation. Extending pivot detection to earlier turns, where recovery distillation would apply, is left to future work. The standard OPD baseline is on-policy distillation from the same Nemotron-3-Super teacher, which distills its token-level distribution on the student's own trajectories \citep{agarwal2024gkd}, and \method{} adds preventive distillation to this objective. Both methods share data, scaffold, and tuning budget, and they train at a constant learning rate of $3 \times 10^{-6}$.

\paragraph{Evaluation.} We evaluate on the $500$ tasks of SWE-Bench Verified using the OpenCode agent scaffold within the NeMo Gym harness. Each task is run once under each of $3$ evaluation seeds in an isolated per-task sandbox, and the agent queries our model through a disaggregated vLLM deployment with separate prefill and decode workers. We sample at temperature $1.0$ with top-$p$ of $0.95$ and report the resolve rate, i.e., the percentage of tasks whose submitted patch passes the associated unit tests, averaged over the three evaluation seeds.

\section{Additional Results}
\label{app:additional_results}

\Cref{tab:ablation_full} reports the component ablations of \Cref{sec:discussion} on every ALFWorld task type. The preventive-only variant sets $K = 0$ and raises the preventive weight from $w_\text{prev} = 0.001$ (\Cref{tab:method_hparams}) to $0.1$. Without recovery distillation, the smaller weight would add only a small pivot-specific signal to the group-relative RL advantage. The larger weight therefore makes preventive only a stronger reference, so its comparison with \method{} tests whether recovery distillation adds to a substantial preventive signal rather than to group-based RL alone. The variant with random pivotal turns injects the same hints as \method{} at randomly chosen turns under a matched budget. Relative to preventive only, removing gold actions from the hints lowers the success rate the most on Pick2, by $9.3\%$. This gap suggests that naming the action is necessary on task types whose decisions cannot be inferred from the observations alone.

\begin{table}[h]
\centering
\small
\setlength{\tabcolsep}{6pt}
\renewcommand{\arraystretch}{1.15}
\begin{tabular}{l@{\hspace{7pt}} ccccccc}
\toprule
\textbf{Variant}
& {\footnotesize Pick} & {\footnotesize Look} & {\footnotesize Clean} & {\footnotesize Heat} & {\footnotesize Cool} & {\footnotesize Pick2} & {\footnotesize\textbf{Avg.}} \\
\midrule
Random pivotal turns
 & 71.4 & \cellcolor{cellsecond}\underline{78.6} & 76.2 & 61.1 & 50.0 & 38.6 & 62.6 \\
Hints w/o gold actions
 & 87.3 & \cellcolor{cellbest}\textbf{82.7} & 90.5 & 63.7 & \cellcolor{cellsecond}\underline{72.9} & 34.1 & 71.9 \\
Reverse-KL recovery
 & 75.0 & 71.4 & 85.7 & 50.0 & 61.5 & 43.4 & 64.5 \\
Recovery only
 & \cellcolor{cellbest}\textbf{87.7} & 63.1 & 66.7 & 40.9 & 53.9 & \cellcolor{cellbest}\textbf{76.2} & 64.7 \\
Preventive only
 & 83.8 & 71.4 & \cellcolor{cellsecond}\underline{93.2} & \cellcolor{cellsecond}\underline{70.2} & \cellcolor{cellsecond}\underline{72.9} & 43.4 & \cellcolor{cellsecond}\underline{72.5} \\
\textbf{\method{}}
 & \cellcolor{cellsecond}\underline{87.6} & 55.9 & \cellcolor{cellbest}\textbf{93.7} & \cellcolor{cellbest}\textbf{72.6} & \cellcolor{cellbest}\textbf{73.9} & \cellcolor{cellsecond}\underline{58.5} & \cellcolor{cellbest}\textbf{73.7} \\
\bottomrule
\end{tabular}
\vspace{4pt}
\caption{\textbf{Full component ablations on ALFWorld with the Qwen3-1.7B student.} Each variant changes one component of \method{} and holds the rest of the training recipe fixed, except that preventive only also raises $w_\text{prev}$ to $0.1$. This table extends \Cref{tab:ablation_components} to every task type. Results are averaged over three seeds.}
\label{tab:ablation_full}
\end{table}

\Cref{fig:ablation_k} presents the validation trends for the recovery-budget ablation in \Cref{sec:discussion}, where $K = 0$ is the preventive-only variant. The choice of $K$ has a substantial effect on validation performance. On WebShop, the score at the end of training differs by up to $19.5\%$ across budgets. More recovery turns, however, do not always improve performance, and increasing the budget from $K = 1$ to $K = 3$ lowers this score by $7.8\%$.

\subsection{Training Compute Overhead}
\label{app:compute}

\method{} changes only training. At test time, the student acts without the teacher model, hints, or recovery, so \method{} adds no computation at inference. During training, preventive and recovery distillation add two computations to group-based RL. The first is the forward pass of the privileged self-teacher, which computes the distillation advantage of \Cref{eq:distill_adv}. The second consists of the \emph{recovery rollouts}, which query the teacher model for a recovery action and sample a recovery response at every recovery turn (\Cref{alg:method}). We do not include pivot detection, whose cost depends mainly on the choice of teacher model and how it is served.

We time both computations on ALFWorld with the Qwen3-1.7B student on one node of four H100 GPUs, where GRPO takes $367.3$ seconds per training step (\Cref{tab:compute}). The self-teacher forward pass is the same kind of computation that self-distillation baselines such as OPSD and RLSD perform. With $K = 1$, recovery rollouts are inexpensive because the only recovery turn starts from the post-mistake state that the rollout has already reached, so it needs no environment replay. Each later recovery turn must instead be reached through environment replay (\Cref{app:method_details}), and the time of recovery rollouts grows by more than ten times from $K = 1$ to $K = 2$. Among the three main benchmarks, only ALFWorld uses more than one recovery turn (\Cref{tab:method_hparams}). On WebShop, where $K = 1$ is the selected budget, recovery rollouts likewise add only $12.2\%$ to the time per training step of GRPO, measured with $16$ GPUs per run over the first $41$ training steps.

\begin{table}[h]
\centering
\small
\setlength{\tabcolsep}{6pt}
\renewcommand{\arraystretch}{1.1}
\begin{tabular}{l cccc}
\toprule
\textbf{Computation} & $K = 0$ & $K = 1$ & $K = 2$ & $K = 3$ \\
\midrule
Self-teacher forward pass & \multicolumn{4}{c}{$17.2$ s ($+4.7\%$)} \\
Recovery rollouts & \na & $28.5$ s ($+7.8\%$) & $328.7$ s ($+89.5\%$) & $395.1$ s ($+107.6\%$) \\
\midrule
\textbf{Total overhead} & $+4.7\%$ & $+12.4\%$ & $+94.2\%$ & $+112.3\%$ \\
\bottomrule
\end{tabular}
\vspace{4pt}
\caption{\textbf{Training compute overhead of \method{} on ALFWorld with the Qwen3-1.7B student.} Each entry is the median time per training step of a computation that preventive and recovery distillation add to group-based RL, with the resulting increase over the time per training step of GRPO in parentheses. The last row sums both increases, and $K = 0$ is the preventive-only variant. The overhead stays small with one recovery turn and grows once later recovery turns require environment replay.}
\label{tab:compute}
\end{table}

\subsection{Recovery Analysis Details}
\label{app:recovery_analysis}

This analysis examines whether trained policies can recover after a pivotal mistake has already been committed, complementing the motivating replay study in \Cref{app:analysis}.

\paragraph{Replay protocol.} We use the same $72$ pivotal mistakes identified by the environment's symbolic oracle in \Cref{sec:evidence}. These oracle-labeled pivotal mistakes are independent of the teacher-detected pivotal turns used in training, so this evaluation does not reward agreement with the teacher's own labels. For each pivotal mistake, we replay the trajectory prefix up to and including the pivotal mistake in a fresh copy of the ALFWorld environment, verify that the reached state matches the recorded one, and let the policy continue on its own for the remainder of the $30$-turn budget. Each prefix is replayed $8$ times per policy, giving $576$ replays per policy.

\paragraph{Policies.} We compare four policies. \textbf{Base} is the untrained Qwen3-8B model. \textbf{Standard OPD} is the best checkpoint of the OPSD run in \Cref{app:analysis} from a sweep over training steps. \textbf{Preventive only} is \method{} trained for $160$ steps with recovery budget $K = 0$, as in \Cref{sec:discussion}. \textbf{\method{}} is the full method trained for $160$ steps with $K = 2$.

\paragraph{Recovery curves (\Cref{fig:case_study_curves}, top).} A replay counts as recovered within $x$ turns when it completes the task and its continuation after the pivotal mistake uses at most $x$ turns. The curve at each value of $x$ is the per-mistake mean averaged over the $72$ pivotal mistakes, so every pivotal mistake contributes equally regardless of how many of its replays succeed. Final recovery rates are $8.3\%$ (Base), $20.3\%$ (Standard OPD), $45.8\%$ (preventive only), and $72.7\%$ (\method{}).

\paragraph{Paired comparison per pivotal mistake.} Because every policy replays the same $72$ pivotal mistakes, we also compare each trained policy with the base model on each pivotal mistake separately. The recovery rate of a pivotal mistake is the fraction of its $8$ replays that recover, and the sign of its difference from the rate of the base model marks the pivotal mistake as improved, worsened, or unchanged. Standard OPD improves the recovery rate on $20$ of the $72$ pivotal mistakes, worsens it on $4$, and leaves $48$ unchanged. The preventive-only variant improves $47$, worsens $6$, and leaves $19$ unchanged. \method{} improves $60$ and worsens none, leaving $12$ unchanged. These counts come from the same replays as the recovery curves, and they show that the higher recovery rate of \method{} extends across most pivotal mistakes rather than coming from a few of them.

\paragraph{Recovery efficiency (\Cref{fig:case_study_curves}, bottom).} Among the replays that recover, we report the average number of turns from the pivotal mistake to task completion, with $95\%$ bootstrap confidence intervals ($2{,}000$ draws, seed $0$). The gray reference bar reports the average remaining optimal trajectory at the post-mistake state, computed over the $72$ pivotal mistakes. Because each policy recovers on a different number and mixture of pivotal mistakes, the efficiency values are conditioned on different sets of replays: $48$ (Base), $117$ (Standard OPD), $264$ (preventive only), and $419$ (\method{}). Despite this difference in conditioning, the per-policy recovered sets have near-identical own-set optimal means ($5.2$--$5.8$ turns), so the common reference bar provides a fair comparison across policies.

\subsection{Divergence from the Self-Teacher Details}
\label{app:recovery_kl}

\Cref{fig:recovery_kl} compares the Qwen3-8B checkpoints of the preventive-only variant ($K = 0$) and a separate \method{} run with $K = 1$ at training steps $140$ and $160$. Each pair of checkpoints scores the ALFWorld training rollouts of the preventive-only variant from a disjoint window of steps, i.e., steps $130$--$149$ for the checkpoints at step $140$ and steps $150$--$160$ for those at step $160$. Both variants therefore score the same responses at the same turns, and every scored state is one that the preventive-only variant visits itself. For each turn, we compute the forward KL divergence from each policy's privileged self-teacher to the policy over the full vocabulary and average it over the response tokens. The self-teacher receives the hints recorded with these rollouts, which include the trajectory-level lesson at every turn and the gold action at candidate turns. The preventive-only variant records no recovery actions, so this divergence uses the recorded hints rather than the recovery-action hint of \Cref{eq:recovery}, and it serves as a proxy for the recovery loss rather than the loss itself (\Cref{lem:kl_bound}). We align each trajectory at its first teacher-detected pivotal turn and exclude turn $0$. The analysis covers $600$ trajectories, $351$ of which contain a teacher-detected pivotal turn. The shaded bands show $95\%$ bootstrap confidence intervals clustered over trajectories ($2{,}000$ draws, seed $0$). In \Cref{fig:recovery_kl}, both students track their self-teachers closely before the pivotal mistake, and the divergence spikes at the pivotal turn, where the hint names the gold action.

\section{Related Works}
\label{app:related_work}

\paragraph{On-policy distillation and the direction of KL.} Sequence-level knowledge distillation trains a student on outputs sampled from the teacher \citep{kim2016sequence, yang2026sfd}. On-policy distillation instead supervises the student on its own samples, and prior work discusses which KL direction suits this setting. GKD compares forward and reverse KL on student-generated sequences \citep{agarwal2024gkd}, and MiniLLM favors reverse KL so that the student does not overestimate regions where the teacher places little mass \citep{gu2024minillm,lu2025onpolicydistillation}. For multi-turn agents, recent methods reweight turns or schedule the distilled trajectory depth \citep{zhou2026turnopd,wang2026tcod,zhong2026sod,zhang2026stepopsd,wang2026agentopsd}, but they still distill only on student-sampled responses. \method{} chooses the direction by where the supervision is needed. Preventive distillation applies reverse KL to a response that the student has already sampled, while recovery distillation applies forward KL to self-teacher responses that the student rarely produces, in the spirit of sequence-level distillation.

\paragraph{Self-correction and learning from failures.} Several lines of work teach language models to correct or learn from their own mistakes. Reflexion stores verbal reflections on failed trials in memory instead of updating the model \citep{shinn2023reflexion}. SCoRe and RISE train models to revise their responses over multiple attempts \citep{kumar2024score,qu2024rise}. For agents, Agent-R splices failed trajectories with correct continuations found by Monte Carlo tree search \citep{yuan2025agentr}. ETO and NAT fine-tune agents on failed trajectories through contrastive pairs and explicit failure markers, respectively \citep{song2024eto,wang2024nat}, and LEMA fine-tunes on mistake corrections written by a stronger model \citep{an2023lema}. These methods rely on outcome rewards or offline revision data, whereas \method{} provides dense teacher supervision at the post-mistake states that the student's own mistakes create.

\paragraph{Credit assignment at key turns.} Outcome rewards reveal little about which turn caused a failure. GiGPO estimates step-level advantages by grouping actions taken from repeated states across rollouts \citep{feng2025gigpo}, and process reward models score intermediate steps \citep{lightman2023prm,wang2024mathshepherd}. PivotRL \citep{yi2026pivotrl} and OPID \citep{yang2026opid} concentrate training on the few turns that determine the outcome. Like standard OPD, however, these methods assign credit only to actions that the student samples, whereas \method{} also distills recovery actions that the student rarely produces after its own pivotal mistakes.

\paragraph{Interactive imitation learning.} Recovery distillation resembles interactive imitation learning, where an expert supervises the learner at the states that the learner itself reaches. DAgger queries the expert at states visited by the learner's policy \citep{ross2011reduction}, and HG-DAgger lets a human expert take over when the learner enters unsafe states \citep{kelly2019hgdagger}. Recovery distillation similarly queries the teacher at the post-mistake states produced by the student's own mistakes. It differs in that the teacher only names the recovery action, while the token-level target comes from the student's own hinted distribution.

\paragraph{Privileged information.} Learning using privileged information provides extra information during training that is unavailable at test time \citep{vapnik2009lupi}, as in asymmetric actor-critic methods whose critic observes the full state \citep{pinto2018asymmetric}. Recent work on language models conditions a self-teacher on privileged information \citep{zhao2026opsd,penaloza2026pid, he2026sdzero, kaur2026rethinking, liu2026mixsd}. \method{} uses the named gold or recovery action as the privileged information and trains the student without it.

\section{Limitations and Open Directions}
\label{app:limitations}

\paragraph{Replayable environments.} Recovery turns after the first are reached by replaying the recorded actions in a copy of the environment (\Cref{app:method_details}), so recovery budgets above $K = 1$ require an environment that reproduces recorded observations exactly. ALFWorld, WebShop, and Search-based QA meet this requirement, and our budget ablation uses up to $K = 3$ on each (\Cref{fig:ablation_k}). Many interactive environments, such as live websites, do not. Extending pivot detection to earlier turns of such long episodes and reaching later recovery turns without exact replay remain open directions.

\paragraph{Dependence on the teacher model.} \method{} trains toward the actions that the teacher model names, so a wrong gold action or recovery action becomes a wrong distillation target. On ALFWorld, pivot detection agrees with the oracle within one turn in $77.8\%$ of failed trajectories on average (\Cref{tab:teacher_oracle}), so in the remaining cases, supervision lands away from the first oracle-labeled pivotal turn. The mismatch check can also flag a reasonable alternative action as pivotal (\Cref{app:method_details}). When the student serves as its own teacher, \method{} remains the best method in \Cref{fig:self_distillation}, but its ALFWorld success rate falls short of that under the stronger teacher (\Cref{tab:main_results}). Estimating the reliability of each named action before distilling it could reduce this dependence.

\paragraph{Tuning and scope of the diagnosis.} The recovery budget $K$ and the recovery weight $w_\text{rec}$ interact, and we select both per benchmark on validation (\Cref{app:hyperparameters}). Applying \method{} to a new benchmark therefore requires a sweep over them. Our diagnosis of pivotal mistakes relies on the symbolic oracle of ALFWorld, which the other benchmarks lack, so on those benchmarks only the teacher model estimates where pivotal turns occur. Finally, every agent sees a history window of at most five turns (\Cref{tab:train_hparams}). The turns that the agents waste after a pivotal mistake in \Cref{sec:evidence} are counted under this window, so part of this waste may come from the limited history.

\newpage

{
  \small
  \bibliographystyle{unsrt}
  \bibliography{ref}
}

\end{document}